\documentclass{article}

\usepackage{microtype}
\usepackage{graphicx}
\usepackage{subcaption}
\usepackage{booktabs} 

\usepackage[preprint]{icml2026}

\usepackage{amsmath}
\usepackage{amssymb}
\usepackage{mathtools}
\usepackage{amsthm}

\theoremstyle{plain}
\newtheorem{theorem}{Theorem}[section]

\newtheorem{lemma}[theorem]{Lemma}

\theoremstyle{definition}
\newtheorem{definition}[theorem]{Definition}

\theoremstyle{remark}
\newtheorem{remark}[theorem]{Remark}
\newtheorem{example}[theorem]{Example}

\usepackage[textsize=tiny]{todonotes}

\icmltitlerunning{Noisy Models For Contrastive Learning}

\usepackage[utf8]{inputenc} 
\usepackage[T1]{fontenc}    
\usepackage{url}            
\usepackage{booktabs}       
\usepackage{amsfonts,bbm}       
\usepackage{nicefrac}       
\usepackage{microtype}      
\usepackage{amsmath}
\usepackage{amssymb}
\usepackage{thmtools}
\usepackage[dvipsnames]{xcolor}
\usepackage{lineno}
\usepackage{times}
\usepackage{multirow}
\usepackage{thm-restate}
\usepackage{enumitem}

\usepackage{tikz}
\usetikzlibrary{calc, angles, quotes, backgrounds, patterns, intersections, positioning}

\definecolor{colX}{RGB}{0, 114, 178}       
\definecolor{colProj}{RGB}{0, 158, 115}    
\definecolor{colPrime}{RGB}{213, 94, 0}    
\definecolor{colSample}{RGB}{100, 100, 100} 

\tikzset{
    boundaryStyle/.style={ultra thick, black},
    secondaryBoundary/.style={thick, black},
    projLine/.style={dashed, thick, gray},
    connectLine/.style={dotted, ultra thick, gray},
    pointStyle/.style={circle, inner sep=0pt, minimum size=6pt},
    sampleSpace/.style={thick, blue!40!gray, fill=blue!2}
}

\DeclareMathOperator{\CE}{CE}

\DeclareMathOperator{\argmax}{argmax}
\DeclareMathOperator{\argmin}{argmin}
\DeclareMathOperator{\expc}{exp}
\DeclareMathOperator{\dist}{dist}
\DeclareMathOperator{\passive}{passive}
\DeclareMathOperator{\act}{active}
\DeclareMathOperator{\proj}{proj}
\DeclareMathOperator{\prop}{prop}
\DeclareMathOperator{\app}{app}
\let\vec\mathbf

\newcommand{\CEmin}{\CE^{d}_{\mathrm{min}}}
\newcommand{\CEnoisy}{\CE^{d, f}_{\app}}
\newcommand{\CEprob}{\CE^{d, f}_{\prop}}
\newcommand{\CEnoisyf}[1]{\CE^{d, #1}_{\app}}
\newcommand{\CEprobf}[1]{\CE^{d, #1}_{\prop}}

\newcommand{\VOL}{\mathrm{VOL}}

\newcommand{\cX}{{\mathcal X}}
\newcommand{\cC}{{\mathcal C}}

\newcommand{\cA}{{\mathcal A}}

\newcommand{\err}{\mathrm{err}}

\newcommand{\reals}{\mathbb{R}}

\newcommand{\thresh}{\cC_{\mathrm{thresh}}}
\newcommand{\HS}{\cC_{\mathrm{HS}}}
\newcommand{\HHS}{\cC_{\mathrm{HHS}}}

\newcommand{\STOP}[3]{\tau_{#1}\!\left(#2 \rightarrow #3\right)}

\newcommand{\fexp}{\mathbbm{e}}
\newcommand{\fpol}{\mathbbm{p}}

\usepackage{xspace}

\begin{document}

\twocolumn[
  \icmltitle{Learning Half-Spaces from Perturbed Contrastive Examples}



  \icmlsetsymbol{equal}{*}

  \begin{icmlauthorlist}
    \icmlauthor{Aryan Alavi Razavi Ravari}{equal,MPI}
    \icmlauthor{Farnam Mansouri}{equal,waterloo,vector}
    \icmlauthor{Yuxin Chen}{chicago}
    \icmlauthor{Valentio Iverson}{waterloo}
    \icmlauthor{Adish Singla}{MPI}
    \icmlauthor{Sandra Zilles}{regina,amii}
  \end{icmlauthorlist}


   \icmlaffiliation{waterloo}{David R. Cheriton School of Computer Science, University of Waterloo, Waterloo, ON, Canada}
  \icmlaffiliation{vector}{Vector Institute, Toronto, ON, Canada}
  \icmlaffiliation{regina}{Department of Computer Science, University of Regina, Regina, SK, Canada}
  \icmlaffiliation{amii}{Alberta Machine Intelligence Institute (Amii), Edmonton, AB, Canada}
  \icmlaffiliation{chicago}{Department of Computer Science, University of Chicago, Chicago, IL, USA}
  \icmlaffiliation{MPI}{Max Planck Institute for Software Systems, Saarbruecken, Germany}

  \icmlcorrespondingauthor{Farnam Mansouri}{f5mansou@uwaterloo.ca}
  \icmlcorrespondingauthor{Sandra Zilles}{sandra.zilles@uregina.ca}

  \icmlkeywords{Contrastive Learning}

  \vskip 0.3in
]



\printAffiliationsAndNotice{\icmlEqualContribution}  

\begin{abstract}
  We study learning under a two-step contrastive example oracle, as introduced by \citet{MansouriSSCZ25}, where each queried (or sampled) labeled example is paired with an additional contrastive example of opposite label. While Mansouri et al.\ assume an idealized setting, where the contrastive example is at minimum distance of the originally queried/sampled point, we introduce and analyze a mechanism, parameterized by a non-decreasing noise function $f$, under which this ideal contrastive example is perturbed. The amount of perturbation is controlled by $f(d)$, where $d$ is the distance of the queried/sampled point to the decision boundary. Intuitively, this results in higher-quality contrastive examples for points closer to the decision boundary. We study this model in two settings: (i) when the maximum perturbation magnitude is fixed, and (ii) when it is stochastic.
  
  For one-dimensional thresholds and for half-spaces under the uniform distribution on a bounded domain, we characterize active and passive contrastive sample complexity in dependence on the function~$f$. We show that, under certain conditions on $f$, the presence of contrastive examples speeds up learning in terms of asymptotic query complexity and asymptotic expected query complexity. 
\end{abstract}

\section{Introduction}

Contrastive information---in the form of two \emph{similar}\/ data points of \emph{opposite}\/ class---is valuable in the context of learning, in various ways. Firstly, pairing a labeled example $(x,\ell(x))$ with a highly similar point $x'$ of opposite label can provide the learner with information on the location of the decision boundary and thus greatly speed up learning, as has been demonstrated in a multitude of application settings such as recommender systems \citep{TanXGLCZ21}, reinforcement learning \citep{HuberDMOA23}, NLP \citep{wangQYGZFS24,MargatinaVBA21}, learning formal languages \citep{Tirnauca09}, or program synthesis \cite{AbateDKKP18}. Secondly, contrastive information can explain the predictions of a machine-learned classifier: If such classifier predicts label $\ell(x)$ for datum $x$, then a datum $x'$ similar to $x$, but of opposite label, suggests that the features in which $x$ and $x'$ differ are crucial for the prediction of the label for $x$ \cite{JiangLRT24}. Thirdly, in causal inference and decision-making problems, contrastive examples are used in the form of counterfactuals, by deriving highly valuable information from simulating ``what-if'' situations that contrast with actual observed situations \citep{AguileraLLR25}. 

Hence, recently the learning theory community has begun to develop and analyze formal models of learning from contrastive examples. In this context, researchers typically chose a very specific learning setting or hypothesis class \citep{alon23,WangSC21,PoulisD17,DasguptaS20,WangSC21}, or studied the effect of specific model or environment aspects in learning, such as the choice of loss function or the provision of negative examples \citep{AshGKM22,HaoChenWGM21}. 

Most relevant to our work is a more generic recent model proposed by \citet{MansouriSSCZ25}, in which an active learner queries the label for a data point $x$, and is in addition provided with a contrastive example right on the decision boundary, namely the closest point $x'$ to $x$ that has a label opposite from  $\ell(x)$.\footnote{Technically, in $\reals^k$, it may happen that there is no closest $x'$ of opposite label. In that case, the contrastive data point $x'$ will be the limit of a Cauchy sequence of points whose label differs from $\ell(x)$, even if this results in $\ell(x')=\ell(x)$.} Crucial is that the learner knows the rule by which contrastive examples are chosen. Mansouri et al.\ showed that this setting makes the learner very powerful, with sample complexity drastically shrinking; for example, only a single query is needed in order to learn a linear half-space. This suggests that Mansouri et al.'s model may, in some cases, trivialize learning.

Moreover, in practice it is often unrealistic to assume that an oracle can provide the learner with exactly the closest point $x'$ to $x$ that has a label opposite from  $\ell(x)$. Thus the drastic improvements in query complexity, due to contrastive examples, may not be attainable in practice either. We therefore propose to study variants of Mansouri et al.'s model, where the contrastive example presented to the learner is a perturbed version of the ``ideal'' one from the original model. Specifically, we assume that the amount of possible perturbation of the contrastive example for a point $x$ may increase with the distance $r$ of $x$ from the decision boundary. The relationship between $r$ and the permissible amount of perturbation is governed by a function $f$, by which we parameterize our model. Intuitively, the closer $x$ is to the decision boundary, the higher the ``quality'' of the contrastive example for $x$. We study two mechanisms based on this idea, a probabilistic one and a non-probabilistic one.

While \citet{MansouriSSCZ25} considered only active learners, which always select the data point $x$ for which both a label and a contrastive example will be provided, we study both active and passive settings. 

Given a concept class $\mathcal{C}$, measures of interest in our study are (a) the number of data points required so that the probability of inferring an $\varepsilon$-approximation of the target concept in $\mathcal{C}$ is at least $1-\delta$, in dependence on $(\varepsilon,\delta)\in(0,1]^2$, and (b) the number of data points required so that the expected value of the error between the conjecture and the target concept in $\mathcal{C}$ is at most $\varepsilon$, in dependence on $\varepsilon\in(0,1]$. 

Since linear half-spaces are popular objects of study in computational learning theory, and many real-world machine learning solutions are based on linear classifiers, we focus our analysis exclusively on classes of (i) one-sided threshold functions in $\reals$, and (ii) linear half-spaces (over a unit ball in $\reals^k$, both with and without homogeneity assumption). Our results complement a rich literature on learning half-spaces \emph{without}\/ contrastive examples; in particular we compare our results to those from studies on learning half-spaces with  membership queries \cite{dasgupta2005analysis, balcan2007margin, BalcanL13, hopkins2020power, diakonikolas2024active} We demonstrate that our models of perturbed contrastive examples often (but not always) yield remarkable improvements (in terms of the measures mentioned above) compared to learning without contrastive examples. The improvements here depend on the choice of the function $f$ that determines the permissible amount of perturbation of the ``ideal'' contrastive examples. At the same time, our models do not trivialize learning of half-spaces (by contrast with Mansouri et al.'s idealized model, which allows for learning half-spaces with just a single contrastive example). Thus, our proposed approach makes a significant step towards practically relevant models of learning from contrastive examples.

\section{Preliminaries}

Let $(\cX, d)$ be a metric space with bounded diameter $\sup_{x, x'} d(x, x')$
and let $\mathcal{U}[\cX]$ be the uniform distribution over $\cX$. In this paper, we assume that $d$ is the Euclidean distance.
Further, let $\ell \subseteq \cX$ be a labeling rule.
For any $C \subseteq \cX$ we define the error of $C$ w.r.t.\ $\ell$ 
as 
$\err(C, \ell) := \Pr_{x \sim \mathcal{U}[\cX]}[ \ell(x) \neq C(x)]$. Here, we implicitly identify a subset of $\cX$ (i.e., a concept) with its binary indicator function.

This paper extends the model for learning with contrastive examples that was proposed by \cite{MansouriSSCZ25}. In said model, a learner trying to identify a labeling rule $\ell$ has access to a two-step oracle. The first step corresponds to an oracle providing a labeled training data point $(x,\ell(x))$ (this case is called \emph{passive learning}) or providing the label $\ell(x)$ for a data point $x\in\cX$ of the learner's choice (\emph{active learning}); the resulting data point provided in this step is called the \emph{primary example}. The second step provides the learner with a \emph{contrastive example}\/ $(x',\ell(x'))$ to the previously selected primary example $(x,\ell(x))$. This step is governed by a selection mechanism $\CE$ that describes how the contrastive examples are chosen; we define various such mechanisms below.

The learner thus collects a sequence of pairs of (labeled) examples $[(x_i, \ell(x_i)),(x'_i, \ell(x'_i))]$ for $1\le i\le m$, where for all $i$, $(x_i, \ell(x_i))$ is a primary example, and $(x'_i, \ell(x'_i))$ is the corresponding contrastive example. Using both primary and contrastive examples, the target of the learner is to output a concept $C \subseteq \cX$  that minimizes $\err(C, \ell)$.


\paragraph{Primary Example Collection.} We analyze two well-studied settings for collecting primary examples (the first step of the oracle). (i) Passive learning: The instances $x_i$ of primary examples are sampled i.i.d.\ from $\mathcal{U}[\cX]$. In this setting the sequence of pairs of examples gathered by the learner is simply an unordered set $S$ of pairs of examples. 
(ii) Active learning: For each $i \in [m]$, the learner chooses an $x_i \in \cX$, for which the primary oracle then provides the label $\ell(x_i)$. The choice of $x_i$ may depend on the sequence of previously seen example pairs $[((x_j, \ell(x_j)),(x'_j, \ell(x'_j))]$, $1\le j< i$, i.e., the learner is adaptive.

\paragraph{Contrastive example collection.} We study three different options for the mechanism by which contrastive examples are selected by the oracle. The first one, proposed in the literature, describes an ``idealized'' setting with an often unrealistically strong oracle; we then propose two variations using a weaker oracle. 
\begin{enumerate}
    \item \textbf{Minimum Distance Model (MDM)} The minimum distance mechanism, denoted by $\CEmin$, was introduced by \cite{MansouriSSCZ25}. Given a primary example $(x,\ell(x)) \in \cX\times \{0,1\}$, the oracle returns any contrastive example $(x',\ell(x'))$ for which $x'$ belongs to 
    $$
    \begin{aligned}
        \CEmin(x):=&\arg\min\{d(x,x')\ |\ x' \in \cX,  \\
        & x'=\lim_{k\rightarrow\infty}x'_k 
         \mbox{ for a Cauchy sequence }\\
        &(x'_k)_k\mbox{ with }\ell(x'_k) \neq \ell(x) \mbox{ for all }k \}\,.
    \end{aligned}    
    $$ 
    \item \textbf{Deterministic Approximate Minimum Distance Model (Deterministic AMDM)} We introduce an approximate version $\CEnoisy$ of  $\CEmin$, parameterized by a non-decreasing function $f: \reals^+ \rightarrow \reals^+$.  Here, the oracle returns a perturbed version of the contrastive example selected by the $\CEmin$ mechanism, where the perturbation is dependent on the distance of the primary example $x$ to the decision boundary of $\ell$. In addition, we impose the constraint that the label $\ell(x')$ of the contrastive example be different from $\ell(x)$. Intuitively, the closer $x$ is to the decision boundary, the higher the quality of the contrastive example $(x',\ell(x'))$, in the sense that $x'$ is not very far from an ``ideal'' example in $\CEmin(x)$.

    More formally, given a primary example $(x,\ell(x))$, the oracle chooses as its contrastive example any $(x',\ell(x'))$ where $x'$ is in 
    $$
    \begin{aligned}
        \CEnoisy(x):=\{x' \in \cX\mid d(x', x^d_{\min}) \leq f(d(x, x^d_{\min})) \\
    \mbox{ for some }x^d_{\min}\in\CEmin(x),\\
    \mbox{ and }\ell(x')\ne\ell(x)\}\,
    \end{aligned}$$

   \item \textbf{Probabilistic Approximate Minimum Distance Model (Probabilistic AMDM)} The $\CEnoisy$ mechanism makes no assumption about how the contrastive example is chosen from the set $\CEnoisy(x)$, enforcing an adversarial analysis in our study below. By contrast, we define a probabilistic version $\CEprob$ of $\CEnoisy$, again parameterized by a non-decreasing function $f: \reals^+ \rightarrow \reals^+$. Formally, for any $x \in \cX$, let $\mathcal{D}_x$ be a distribution over points in $\cX$ whose labels differ from $\ell(x)$, such that $\mathbb{E}_{x' \sim \mathcal{D}_x}[d(x', x^d_{\min})] \leq f(d(x, x^d_{\min}))$ for some $x^d_{\min}\in\CEmin$. In the $\CEprob$ mechanism, the oracle independently samples $x'$ from $\mathcal{D}_x$, and provides the contrastive example $(x',\ell(x'))$ for $(x,\ell(x))$. 
\end{enumerate}

For all settings, we assume realizability, i.e., the learner outputs hypotheses from a concept class $\cC$ over $\cX$, such that $\ell\in\cC$. In the following, we formally define various sample complexity measures that we aim to study.

\begin{definition}
Let $\cC$ be a concept class over $\cX$ and let $\CE$ refer to a mechanism for selecting contrastive examples. Let $\mathcal{A}$ be a passive (resp.\ active) learning algorithm.
The \emph{contrastive sample complexity} of $\mathcal{A}$ on $\cC$ with respect to $\CE$, denoted by 
$\mathcal{M}_{\Pr}[\cC, \CE, \mathcal{A}]$, is the function that maps any $(\varepsilon,\delta)\in(0,1]^2$ to the minimal $m^* \in \mathbb{N}$ such that, for every $\ell \in \cC$ and every $m \geq m^*$, we have
\[
\Pr\!\left[ \err\big(C_{\mathcal{A},m}, \ell\big) > \varepsilon \right] < \delta\,,
\]
where $C_{\mathcal{A},m}$ is the final hypothesis output by $\mathcal{A}$ after interacting with the two-step oracle for $m$ steps. The passive (resp.\ active) contrastive sample complexity of $\cC$ with respect to $\CE$, denoted by $\mathcal{M}_{\Pr,\passive}[\cC, \CE]$ (resp.\ $\mathcal{M}_{\Pr,\act}[\cC, \CE]$), 
is defined by 
$$
\mathcal M_{\Pr,\mathrm{passive}}[\mathcal C,\mathrm{CE}](\varepsilon,\delta)
:=\inf_{\mathcal{A}~\mathrm{passive}}\mathcal M_{\Pr}[\mathcal C,\mathrm{CE},\mathcal{A}](\varepsilon,\delta),
$$
$$
\mathcal M_{\Pr,\mathrm{active}}[\mathcal C,\mathrm{CE}](\varepsilon,\delta)
:=\inf_{\mathcal{A}~\mathrm{active}}\mathcal M_{\Pr}[\mathcal C,\mathrm{CE},\mathcal{A}](\varepsilon,\delta).\footnote{For active learning with deterministic learning algorithms, and $\mathrm{CE}\in\{\CEmin,\CEnoisy\}$, the dependency on $\delta$ can be dropped, since there is no stochasticity in the learning process. We choose to ignore this, for consistency of presentation.}
$$
\end{definition}


\begin{definition}
Let $\cC$ be a concept class over $\cX$ and let $\CE$ refer to a mechanism for selecting contrastive examples. Let $\mathcal{A}$ be a passive (resp.\ active) learning algorithm.
The \emph{expected contrastive sample complexity} of $\mathcal{A}$ on $\cC$ with respect to $\CE$, denoted by 
$\mathcal{M}_{\expc }[\cC, \CE, \mathcal{A}]$, is the function that maps any $\varepsilon \in(0,1]$ to the minimal $m^* \in \mathbb{N}$ such that, for every $\ell \in \cC$ and every $m \geq m^*$, we have
\[
\mathbb{E} \!\left[ \err\big(C_{\mathcal{A},m}, \ell\big) \right] \leq \varepsilon.
\]
where $C_{\mathcal{A},m}$ is the final hypothesis output by $\mathcal{A}$ after interacting with the two-step oracle for $m$ steps. The passive (resp.\ active) expected contrastive sample complexity of $\cC$ with respect to $\CE$, denoted by $\mathcal{M}_{\expc,\passive}[\cC, \CE]$ (resp.\ $\mathcal{M}_{\expc,\act}[\cC, \CE]$), 
is defined by 
$$
\mathcal M_{\expc,\mathrm{passive}}[\mathcal C,\mathrm{CE}](\varepsilon)
:=\inf_{\mathcal{A}~\mathrm{passive}}\mathcal M_{\expc}[\mathcal C,\mathrm{CE},\mathcal{A}](\varepsilon),
$$
$$
\mathcal M_{\expc,\mathrm{active}}[\mathcal C,\mathrm{CE}](\varepsilon)
:=\inf_{\mathcal{A}~\mathrm{active}}\mathcal M_{\expc}[\mathcal C,\mathrm{CE},\mathcal{A}](\varepsilon).
$$
\end{definition}

In this work, we study three concept classes. 
Let $\mathcal{B}(\vec x,r) := \{\vec z \in \reals^k : \|\vec z-\vec x\| \le r\}$ denote the (closed) Euclidean ball of radius $r$ centered at $\vec x$.
 
(a) One-dimensional threshold functions over $\cX=[0,1]$:
\[
\thresh := \left\{\, \mathbbm{1}\{x \le \theta\} \;\middle|\; \theta \in [0,1] \,\right\}.
\]
(b) Homogeneous half-spaces over $\cX=\mathcal{B}(\vec 0_k,\tfrac{1}{2})$, where $\vec 0_k$ is the all-zero vector in $\reals^k$:
\[
\HHS := \left\{\, \mathbbm{1}\{\langle \omega, \vec x \rangle \ge 0\} \;\middle|\; \omega \in \reals^k \,\right\}.
\]
(c) Half-spaces over $\cX=\mathcal{B}(\vec 0_k, \tfrac{1}{2})$:
\[
\HS := \left\{\, \mathbbm{1}\{\langle \omega, \vec x \rangle \ge b\} \;\middle|\; \omega \in \reals^k,\; b \in \reals \,\right\}.
\]

All of the results in this paper are stated for a general choice of noise function $f$. 
However, Table~\ref{tab:active} provides an overview of our active learning sample-complexity bounds for two representative choices of $f$ (polynomial $f$ and exponential $f$). 
Detailed calculations for these specific choices are deferred to Appendix~\ref{apdx:different-f}. Also, due to space constraints, details of all the proofs (except for the proof of Theorem~\ref{thm:active-noisy-min-upper-HS}) are deferred to the Appendix.
Further, Table~\ref{tab:passive} also summarizes our passive learning bounds for a general choice of $f$.

\begin{table*}[hbt!]
    \caption{Summary of our sample complexity bounds for contrastive active learning results with respect to $\CEnoisy$ and $\CEprob$. Results are stated for (A) $1$D-thresholds, (B) homogeneous half-spaces, and (C) half-spaces,
    and for two representative classes of noise functions $f$. 
    }
    \begin{subtable}[t]{\linewidth}
        \caption{Active contrastive sample complexity with respect to $\CEnoisy$. We additionally compare to membership query complexity and to the minimum distance model (corresponding to $f(r)=0$). The membership query bounds are due to existing results  \cite{BalcanL13,hopkins2020power}.} 
        \centering
        \begin{tabular}{|l|c|c|c|c|}
        \hline
             & \multicolumn{1}{c|}{\textbf{Non-contrastive}} 
             & \multicolumn{3}{c|}{\textbf{Active Contrastive Queries}} \\
            \cline{2-5}
            & Membership Query & $\frac{r^{1 + c'}}{4} \leq f(r) \leq \frac{r^{1 + c}}{4}$ & $f(r) = \frac{e^{-1/r}}{4}$ & $f(r) = 0$ \\\hline
             (A) 1D-Thresholds  & $\Theta(\log \frac{1}{\varepsilon})$ &  $\Theta(\log \log \frac{1}{\varepsilon})$ & $\Theta(\log^* \frac{1}{\varepsilon})$ & 1 \\ \hline
              (B) Homogeneous Half-spaces & $\Theta(k \log \frac{1}{\varepsilon})$  & 1 & 1 & 1 \\ \hline
              \multirow{2}{*}{(C) Half-spaces} & \multirow{2}{*}{$\Omega \left( \left(\frac{1}{\varepsilon}\right)^{\frac{k - 1}{k + 1}} \right)$ }  &  $\Omega (\log \log \frac{1}{\varepsilon} -  \log k)$ & $\Omega (\log^* \frac{1}{\varepsilon} -  \log^* k)$ & \multirow{2}{*}{1} \\
             &  & $ O (\log \log \frac{1}{\varepsilon} + \log k)$ & $O (\log^* \frac{1}{\varepsilon} +  \log^* k)$&  \\
             \hline
            \end{tabular}

        \label{tab:noisy-active}
    \end{subtable}
    
    \begin{subtable}[t]{\linewidth}
        \caption{Overview of our active learning results with respect to $\CEprob$. We report three quantities:
(i) the active contrastive sample complexity $\mathcal{M}_{\Pr,\act}$,
(ii) the active expected contrastive sample complexity $\mathcal{M}_{\exp,\act}$,
and (iii) the expected sample requirement for accuracy 
$\varepsilon$ of algorithm $\mathcal{N}_{\act}$ (defined formally in Section~\ref{sec:active-prob-acc-req}).
}
        \centering
        \begin{tabular}{|lc|c|c|}
            \hline
            \multicolumn{2}{|c|}{}          
            & $\frac{r^{1 + c'}}{4} \leq f(r) \leq \frac{r^{1 + c}}{4}$                                             & $f(r) = \frac{e^{-1/r}}{4}$                                                              \\ \hline
            \multicolumn{1}{|l|}{\multirow{4}{*}{(A) 1D-Thresholds}} & \multirow{2}{*}{$\mathcal{M}_{\Pr, \act}$} 
            & $\Omega\left(\log \log \frac{1}{\varepsilon} + \sqrt{\log \frac{1}{\delta}} \right)$                              & $\Omega \left(\log^*\frac{1}{\varepsilon} + \log \log \frac{1}{\delta} \right)$                             \\
            \multicolumn{1}{|c|}{}                           &      & $O \left(\log \log \frac{1}{\varepsilon} + \log \frac{1}{\delta} \right)$      & $O \left(\log^*\frac{1}{\varepsilon} + \log \frac{1}{\delta}  \right)$      \\ \cline{2-4} 
            \multicolumn{1}{|c|}{}                            & $\mathcal{M}_{\exp, \act}$  
            & $\Theta \left(\sqrt{\log \frac{1}{\varepsilon}}\right)$                           & $\Theta \left(\log \log \frac{1}{\varepsilon}\right)$                          \\ \cline{2-4} 
            \multicolumn{1}{|c|}{}                            & $\mathcal{N}_{\act}$ 
            & $O \left(\log \log \frac{1}{\varepsilon}\right)$                                  & $O \left(\log^* \frac{1}{\varepsilon}\right)$                                  \\ \hline
             \multicolumn{1}{|l|}{ (B) Homogeneous Half-spaces} & $\mathcal{M}_{\Pr, \act}$ & 1 & 1 \\ \hline 
            \multicolumn{1}{|l|}{\multirow{3}{*}{(C) Half-spaces}} & \multirow{2}{*}{$\mathcal{M}_{\Pr, \act}$} 
            & $\Omega\left(\log \log \frac{1}{\varepsilon} - \log k \right)$                    & $\Omega\left(\log^* \frac{1}{\varepsilon} - \log^* k \right)$                  \\
            \multicolumn{1}{|c|}{}  &  
            & $O \left(\log \log \frac{1}{\varepsilon} + \log k + \log \frac{1}{\delta}\right)$ & $O\left(\log^*\frac{1}{\varepsilon} + \log^* k + \log \frac{1}{\delta}\right)$ \\ \cline{2-4} 
            \multicolumn{1}{|c|}{}                            & 
            $\mathcal{N}_{\act}$
            & 
            $O\left(\log \log \frac{1}{\varepsilon} + \log k \right)$
            & 
            $O\left(\log^* \frac{1}{\varepsilon} + \log^* k \right)$
            \\ \hline
        \end{tabular}
        \label{tab:prob-active}               
    \end{subtable}

    \label{tab:active}
\end{table*}

\begin{table*}[hbt!]
        \caption{Summary of our passive contrastive learning results. We report (i) the passive contrastive sample complexity $\mathcal{M}_{\Pr,\passive}$ with respect to $\CEnoisy$ and (ii) the passive expected contrastive sample complexity $\mathcal{M}_{\exp,\passive}$ with respect to $\CEprob$. Here $g$ is defined by $g(x)=\frac{f(x)}{x}$.
        }
        \centering
        \begin{tabular}{|l|c|c|}
        \hline
        & $\mathcal{M}_{\Pr, \passive}(\cC, \CEnoisy)$  & $\mathcal{M}_{\exp, \passive}(\cC, \CEprob)$ \\ \hline
        (A) 1D-Thresholds              & \begin{tabular}[c]{@{}l@{}}$O \!\left(\frac{\log(1/\delta)}{\min(f^{-1}(\varepsilon), \varepsilon)}\right)$\\ $\Omega\!\left(\frac{\log(1/\delta)}{\min\left(f^{-1}\!\left(2 \varepsilon \right), \varepsilon \right)}\right)$\end{tabular} & \begin{tabular}[c]{@{}l@{}}$O\!\left(\frac{\log(1/\varepsilon)}{\min(f^{-1}(\varepsilon / 2), \varepsilon / 2)}\right)$\\ $\Omega \!\left(\frac{1}{\min(f^{-1}(4\varepsilon), 2\varepsilon)}\right)$\end{tabular} \\ \hline
        (B) Homogeneous Half-spaces & $O \!\left(\frac{\log(1/\delta)}{g^{-1}\left(\pi \varepsilon \right)}\right)$   & $O \!\left(\frac{\log(1/\varepsilon)}{g^{-1}\left(\pi \varepsilon / 2 \right)}\right)$  \\ \hline
        \end{tabular}

        \label{tab:passive}
    \end{table*}

\paragraph{Further Notation.}
Given any measurable $V \subseteq \reals^k$, we denote by $\VOL(V)$ the volume of $V$. For a vector $\vec x \in \reals^k$ and a halfspace $C(\vec z) := \mathbbm{1}\{\langle \omega, \vec z \rangle \leq b\}$, where $\omega \in \reals^k$ and $b \in \reals$, define $\dist(\vec x, C) := \frac{|\langle \omega, \vec x \rangle - b|}{\|\omega\|}$ to be the distance between $\vec x$ and the boundary hyperplane of $C$.
Moreover, for a function \( f : \reals^{+} \to \reals^{+} \) and parameters \( u,b \ge 0 \), let
\[
\STOP{f}{u}{b}
\;:=\;
\min \bigl\{ n \ge 0 \;\big|\; f^{(n)}(u) \le b \bigr\},
\]
where \( f^{(n)} \) denotes the \(n\)-fold composition of \(f\) with itself, with the convention that
\( f^{(0)}(u) := u \).
Intuitively, $\STOP{f}{u}{b}$ is the number of times one must iterate the contraction $f$ to reduce it from $u$ to $b$. If no \( n \) with $f^{(n)}(u) \le b$ exists, we define \( \STOP{f}{u}{b} := \infty \).


\section{Deterministic Approximate Minimum Distance Model} \label{sec:noisy}
\subsection{Active Contrastive Sample Complexity} \label{sec:active-noisy}
In this subsection, we study the active contrastive sample complexity 
with respect to the deterministic AMDM.  We begin by analyzing threshold functions. Under membership queries, if at time $t$ the version space for the target threshold is an interval of length $r$, an active learner can always query the midpoint of this interval to shrink the version space to an interval of length $r/2$. In the following Theorem, we show that under the deterministic AMDM, each query can shrink the version space to an interval of length at most $\tilde f(r)$, where $\tilde f$ is defined in Equation~\eqref{eq:active-noisy-min-thresh}.

\begin{restatable}{theorem}{ActiveNoisyThresh}
    \label{thm:active-noisy-min-thresh}
    Let 
    $f: \reals^{\geq 0} \to \reals^{\geq 0}$ be a non-decreasing and invertible function \footnote{
Invertibility assumptions are made only for ease of presentation. 
Without invertibility, all proofs continue to hold by replacing every $h^{-1}(r)$ for any function $h$ with 
$\inf\{x \ge 0 : h(x) \ge r\}$.}. 
Define $\tilde f: \reals^{\geq 0} \to \reals^{\geq 0}$ by
    \begin{equation}\label{eq:active-noisy-min-thresh}
        \tilde f(r) := f \circ (f + \mathbbm I)^{-1}\!\left(\tfrac{r}{2}\right),
    \end{equation}
    where $\mathbbm I$ denotes the identity function. 
    Then 
    \[
        \mathcal{M}_{\Pr, \act}\left[\thresh, \CEnoisy\right]( \varepsilon, \delta)
        = \STOP{\tilde f}{1}{2\varepsilon}\,.
    \]
   
\end{restatable}

Before moving to the analysis of (homogeneous) half-spaces we need to introduce a critical lemma.



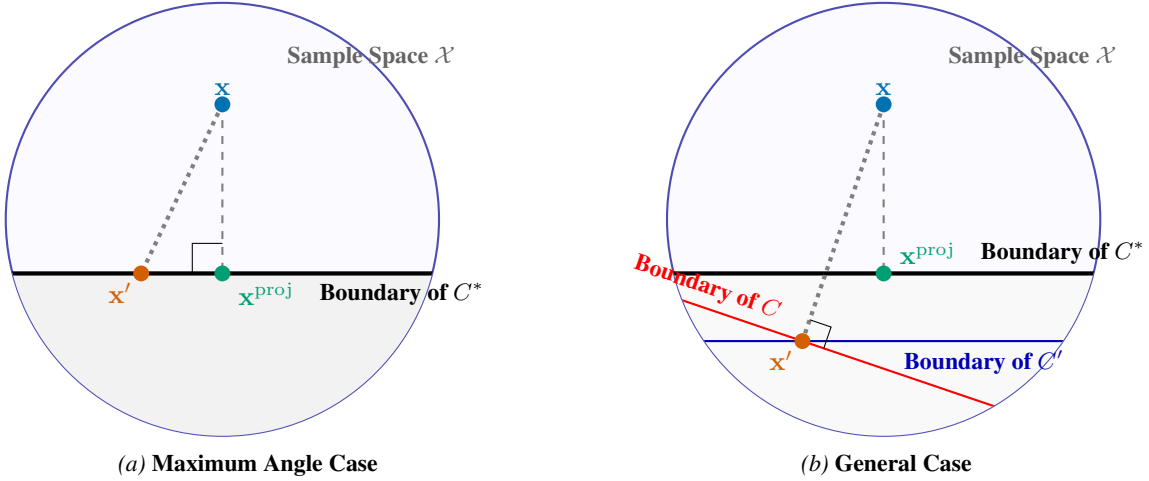
\begin{figure*}[hbt!]
    \centering
    
    \begin{subfigure}[b]{0.49\linewidth}
        \centering
        \begin{tikzpicture}[scale=0.9]
            \draw[sampleSpace] (0, 0.8) circle (3.2cm);
            \node[colSample, font=\bfseries \small] at (2.2, 3.2) {Sample Space $\mathcal{X}$};

            \coordinate (Origin) at (0, 0);       
            \coordinate (X) at (0, 2.5);          
            \coordinate (Xprime) at (-1.2, 0);    

            \begin{scope}
                \clip (0, 0.8) circle (3.2cm);
                \fill[gray!10] (-4, 0) rectangle (4, -3);
                \draw[boundaryStyle] (-4, 0) -- (4, 0);
            \end{scope}
            
            \node[black, below left, font=\bfseries \small] at (4, 0) {Boundary of $C^*$};

            \draw[projLine] (X) -- (Origin);
            \draw[connectLine] (X) -- (Xprime);

            \node[pointStyle, fill=colX] at (X) {};
            \node[above, colX, font=\bfseries] at (X) {$\vec x$};

            \node[pointStyle, fill=colProj] at (Origin) {};
            \node[below right, colProj, font=\bfseries] at (0.1, 0) {$\vec x^{\proj}$};

            \node[pointStyle, fill=colPrime] at (Xprime) {};
            \node[below left, colPrime, font=\bfseries] at (Xprime) {$\vec x'$};

            \pic [draw, angle radius=0.4cm] {right angle = X--Origin--Xprime};
        \end{tikzpicture}
        \caption{\textbf{Maximum Angle Case}}
        \label{fig:circ_max}
    \end{subfigure}
    \begin{subfigure}[b]{0.49\linewidth}
        \centering
        \begin{tikzpicture}[scale=0.9, font=\footnotesize]
            \draw[sampleSpace] (0, 0.8) circle (3.2cm);
            \node[colSample, font=\bfseries \small] at (2.2, 3.2) {Sample Space $\mathcal{X}$};

            \coordinate (Origin) at (0, 0);       
            \coordinate (X) at (0, 2.5);          
            \coordinate (Xprime) at (-1.2, -1.0); 
            \pgfmathsetmacro{\nslope}{-0.34}

            \begin{scope}
                \clip (0, 0.8) circle (3.2cm);
                \fill[gray!5] (-4, 0) rectangle (4, -3);
                
                \draw[boundaryStyle] (-4, 0) -- (4, 0);
                
                \draw[secondaryBoundary, blue!70!black] (-4, -1.0) -- (4, -1.0);
                


                \draw[secondaryBoundary, red] 
                    ($(Xprime) + (-4, {-4*\nslope})$) -- ($(Xprime) + (6, {6*\nslope})$);
            \end{scope}

            \node[black, above left, font=\bfseries \small] at (4, 0) {Boundary of $C^*$};
            \node[blue!70!black, below left, font=\bfseries \small] at (2.8, -1.0) {Boundary of $C'$};
            \node[red, above, font=\bfseries \small, rotate=-19] at ($(Xprime) + (-1.5, {-1.5*\nslope})$) {Boundary of $C$};

            \draw[projLine] (X) -- (Origin);
            \draw[connectLine] (X) -- (Xprime);

            \coordinate (AuxOnC) at ($(Xprime) + (1, \nslope)$);
            \pic [draw, thin, angle radius=0.3cm] {right angle = X--Xprime--AuxOnC};

            \node[pointStyle, fill=colX] at (X) {};
            \node[above, colX, font=\bfseries] at (X) {$\vec x$};

            \node[pointStyle, fill=colProj] at (Origin) {};
            \node[above right, colProj, font=\bfseries] at (0.1, 0) {$\vec x^{\proj}$};

            \node[pointStyle, fill=colPrime] at (Xprime) {};
            \node[below left, colPrime, font=\bfseries] at (Xprime) {$\vec x'$};
        \end{tikzpicture}
        \caption{\textbf{General Case}}
        \label{fig:circ_general}
    \end{subfigure}
    \caption{Geometric interpretation of Lemma~\ref{lem:err-dist}. 
Given a primary point $\vec x \in \mathcal{X}$, its contrastive example $\vec x'$, 
and a target half-space $C^*$, Panel~\eqref{fig:circ_max} illustrates the configuration maximizing the angle $\angle(\vec x',\vec x,\vec x^{\proj})$. 
Panel~\eqref{fig:circ_general} depicts the hyperplane $C$ induced by the pair $(\vec x,\vec x')$ in Lemma~\ref{lem:err-dist}(ii) in the general case, along with the auxiliary half-space $C'$ that is parallel to $C^*$ and passes through $\vec x'$. 
The half-space $C'$ is used to decompose $\err(C,C^*)$; further details are deferred to Appendix~\ref{apdx:active-noisy}.
}
\label{fig:lem-err-dist}
\end{figure*}


    
       




\begin{restatable}{lemma}{ErrDist} 
     \label{lem:err-dist}
    Let $\cX = \mathcal{B}\!\left(\vec 0_k,\tfrac{1}{2}\right)$, $\vec x \in \cX$, $\omega \in \reals^k$, and $b \in \reals$.
Consider the half-space
\(
C^*(\vec z) := \mathbbm{1}\{\langle \omega,\vec z\rangle \ge b\}.
\)
Let
\(
\vec x^{\proj}
:= \vec x - \frac{\langle \omega,\vec x\rangle - b}{\|\omega\|^{2}}\,\omega
\)
denote the orthogonal projection of $\vec x$ onto the boundary hyperplane of $C^*$.
Let $\vec x' \in \cX$ be any point such that $C^*(\vec x') \neq C^*(\vec x)$. Then

\begin{enumerate}[label=(\roman*)]
    \item The angle $\angle(\vec x',\vec x,\vec x^{\proj})$ is at most
    \hspace*{-\leftmargin}\begin{minipage}{\linewidth}
    \[
    \min\!\left\{
    \tan^{-1}\!\left(\frac{\|\vec x' - \vec x^{\proj}\|}{\|\vec x - \vec x^{\proj}\|}\right),
    \ \cos^{-1}\!\left(\frac{\|\vec x - \vec x^{\proj}\|}{\|\vec x - \vec x'\|}\right)
    \right\}.
    \]
    \end{minipage}

    \item Let $C$ be the half-space whose boundary hyperplane is perpendicular to $\vec x' - \vec x$ and passes through $\vec x'$, and such that $C(\vec x)=C^*(\vec x)$.
    Then
    \[
    \err(C, C^*) \le \frac{2^{k}\,\|\vec x' - \vec x^{\proj}\|}{\|\vec x - \vec x^{\proj}\|}.
    \]

    \item Suppose $b=0$. Let $C^0$ be a homogeneous half-space whose boundary hyperplane is perpendicular to $\vec x'-\vec x$ such that $C^0(\vec x)=C^*(\vec x)$. Then
    \[
    \err(C^0,C^*) \le \frac{\|\vec x' - \vec x^{\proj}\|}{\pi\,\|\vec x - \vec x^{\proj}\|}.
    \]
\end{enumerate}
\end{restatable}

Figure~\ref{fig:lem-err-dist} illustrates Lemma~\ref{lem:err-dist}(ii). Statements (i) and (iii) can be illustrated analogously.
Let $\vec x$ be a point at distance $r$ from a target half-space. 
Lemma~\ref{lem:err-dist} shows that, given $\vec x$ and its contrastive example from the deterministic AMDM oracle, a learner can find a concept with error at most (i) $\frac{f(r)}{\pi r}$ if the target classifier is homogeneous, and (ii) $\frac{2^k f(r)}{r}$ otherwise. 
Thus, to obtain a concept with small error, it suffices to find a point that is close to the target concept. 
In particular, for homogeneous half-spaces, any point sufficiently close to the center has this property. 
In the following corollary, we show that this makes active contrastive learning of homogeneous half-spaces trivial even with respect to the more general probabilistic AMDM.

\begin{restatable}{corollary}{ActiveProbHHS}
    \label{cor:active-prob-HHS}
    Let $f:\reals^{\geq 0} \rightarrow \reals^{\geq 0}$ be a non-decreasing function such that $f(x) \leq x^{1 + c}$ where $c > 0$. Then 
    \small
    $$
    \begin{aligned}
        & \mathcal{M}_{\Pr, \act}\left[\HHS, \CEnoisy\right]( \varepsilon, \delta) =  
        \\& \quad 
        \mathcal{M}_{\Pr, \act}\left[\HHS, \CEprob\right]( \varepsilon, \delta) = 1\,.
    \end{aligned}
    $$
\end{restatable}



The proof of the following theorem shows that for non-homogeneous half-spaces, one can find a point close to the decision boundary by first querying $\vec 0$, and then repeatedly querying the contrastive example of the previous query. 

 

\begin{theorem} \label{thm:active-noisy-min-upper-HS}
    Let $f:\reals^{\geq 0} \rightarrow \reals^{\geq 0}$ be a non-decreasing function such that $f(x) \leq x^{1 + c}$ where $c > 0$. Then 
    \[
        \mathcal{M}_{\Pr, \act}\left[\HS, \CEnoisy\right]( \varepsilon, \delta)
        \leq \STOP{f}{\frac{1}{2}}{\varepsilon'}\,,
    \]
    where $\varepsilon' = \left(\frac{\varepsilon}{2^{k}}\right)^{\frac{1}{c}} $.
\end{theorem}

\begin{proof}
Denote $m = \STOP{f}{\frac{1}{2}}{\varepsilon'}.$ Suppose at time $t = 1$, the learner queries $x_1 = \vec 0_k$ and for $t \geq 2$ the learner queries $x_t = x'_{t - 1}$. For each time $t$, let $x^{\proj}_t$ be the projection of $x_t$ on $\ell$. By definition 
$$
\begin{aligned}
    \dist(x_t, \ell) & \leq \|x_t - x^{\proj}_{t - 1} \| \leq f(\|x_{t - 1} - x^{\proj}_{t - 1}\|) \\ & = f(\dist(x_{t - 1, \ell}))\,.
\end{aligned}
$$
Therefore, after $m$ queries, we have $\dist(x_m, \ell) \leq \varepsilon'$. Suppose the learner returns the half-space $C$ corresponding to the hyperplane perpendicular to $x'_m - x_m$ at $x'_m$. Due to Lemma~\ref{lem:err-dist} (ii) we have
$$
\err(C, \ell) \leq \frac{2^{k} f(\varepsilon')}{\varepsilon' } \leq 2^{k} {\varepsilon'}^{c} = \varepsilon.
$$
This completes the proof.
\end{proof}

Observe that the active contrastive sample complexity of half-spaces can be easily bounded from below by the active contrastive sample complexity of thresholds.

\begin{restatable}{theorem}{ActiveNoisyLowerHS}
    \label{thm:active-noisy-min-lower-HS}
    For any non-decreasing function $f:\reals^{\geq 0} \rightarrow \reals^{\geq 0}$, we have 
    \[
    \begin{aligned}
        & \mathcal{M}_{\Pr, \act}[\HS, \CEnoisy]\left(\varepsilon, \delta\right)
        \\
        & \;\; 
        \geq \mathcal{M}_{\Pr, \act}[\thresh, \CEnoisy] \left(2^{k}\varepsilon, \delta\right).
    \end{aligned}
    \]
\end{restatable}



\subsection{Passive Contrastive Sample Complexity} \label{sec:passive-noisy}
This subsection focuses on the passive contrastive sample complexity of one-dimensional threshold functions and homogeneous half-spaces with respect to the deterministic AMDM. 

Again we begin with threshold functions. The key intuition is that it suffices to identify a point very close to the target threshold (here the learner uses the largest positive example). If this point has distance $r$ to the target threshold, then by the deterministic AMDM its contrastive example lies within distance at most $f(r)$. Thus, once $r$ is small, a single contrastive example brings the learner even closer, yielding improved sample complexity.

\begin{restatable}{theorem}{PassiveNoisyThresh}  \label{thm:passive-noisy-min-thresh}
    For any non-decreasing and invertible function $f: \reals^{\geq 0} \to \reals^{\geq 0}$, 
    and for all $\delta \in (0, 1]$ and $\varepsilon \leq \min\!\left( \frac{f(\tfrac{1}{4})}{2}, \tfrac{1}{4}\right)$, 
    we have
    \small{
    \[
    \begin{aligned}
        \mathcal{M}_{\Pr, \passive}\left[\thresh, \CEnoisy\right](\varepsilon, \delta) &\in O \!\left(\frac{\log(1/\delta)}{\min(f^{-1}(\varepsilon), \varepsilon)}\right), \\
         \mathcal{M}_{\Pr, \passive}\left[\thresh, \CEnoisy\right](\varepsilon, \delta) 
          &\in \Omega\!\left(\frac{\log(1/\delta)}{\min\left(f^{-1}\!\left(2 \varepsilon \right), \varepsilon \right)}\right).
    \end{aligned}  
    \]}

\end{restatable}

The core idea behind the threshold analysis also applies to linear separators in general: if the learner can find a point in the sample that is close to the target decision boundary (here, it chooses the example in the data set with the smallest distance to its contrastive example), then its contrastive example provides enough geometric information to construct a hypothesis with bounded error via Lemma~\ref{lem:err-dist}, both for homogeneous and non-homogeneous half-spaces.

However, due to the $2^k$ factor in the error bound of Lemma~\ref{lem:err-dist}(ii) for non-homogeneous half-spaces, this approach yields passive learning guarantees (in contrast to active learning) that scale exponentially with the dimension $k$. This is undesirable, since standard passive learning bounds for half-spaces typically depend only linearly on $k$. Fortunately, Lemma~\ref{lem:err-dist}(iii) allows us to circumvent this issue for homogeneous half-spaces. In fact, our analysis for homogeneous half-spaces is completely dimension-free.





\begin{restatable}{theorem}{PassiveNoisyHS} \label{thm:passive-noisy-min-HS}
    Let $f: \reals^{\geq 0} \to \reals^{\geq 0}$ be such that $f(x) \leq x^2$ and $g(x) := \frac{f(x)}{x}$ is non-decreasing and invertible. Then 
    \[
        \mathcal{M}_{\Pr, \passive}\left[\HHS, \CEnoisy\right](\varepsilon, \delta)
        \;\in\; O \!\left(\frac{\log(1/\delta)}{g^{-1}\left(\pi \varepsilon \right)}\right).
    \]
    
\end{restatable}

\section{Probabilistic Approximate Minimum Distance Model} \label{sec:prob}


\subsection{Active Expected Sample Complexity} \label{sec:active-prob-expected}
Next, we study the active expected contrastive sample complexity 
of one-dimensional threshold functions with respect to the probabilistic AMDM. As can be seen from Table~\ref{tab:prob-active}, in this case for both representative choices of the noise function $f$, the probabilistic AMDM exhibits a noticeable decrease in sample complexity compared to the deterministic AMDM, even for the class of threshold functions. 
Nevertheless, both attain a non-trivial sample complexity improvement over membership queries alone. 

Crucially, our analysis requires  $g(x) := \frac{f(x)}{x}$ to be non-decreasing and convex. Therefore, in the following lemma we first state a basic property of convex functions.

\begin{restatable}{lemma}{convrv}\label{lem:conv-rv}
Let $Z$ be a random variable supported on $[a, b]$. Then for any convex function $f:\reals \rightarrow \reals$ we have 
    $$
    \mathbb{E}[f(Z)] \leq \frac{\mathbb{E}[Z] - a}{b - a} f(b) + \frac{b - \mathbb{E}[Z]}{b - a} f(a).
    $$
\end{restatable}

The key idea of our analysis is as follows. Suppose that at some time $t$ the active learner has a version space of size $r_t$. We show that by querying the midpoint of the version space and one of its endpoints, and using their contrastive examples, the learner obtains a new version space of size $r_{t+1} \in [0, r_t/2]$, where $\mathbb{E}[r_{t+1}] \leq f(r_t/2)$. Lemma~\ref{lem:conv-rv} then allows us to bound the expected size of the final version space.

\begin{restatable}{theorem}{ActiveProbThreshUpper} \label{thm:active-prob-thresh-upper}
Let $f: \reals^{\geq 0} \to \reals^{\geq 0}$ be such that $f(x) \leq x$ and  $g(x) := \frac{f(x)}{x}$ 
is non-decreasing and convex. Then 
\[
\begin{aligned}
    & \mathcal{M}_{\expc, \act}\left[\thresh, \CEprob\right](\varepsilon) \leq \\
    & \quad 
    \min \left\{ n \ge 0 \,\middle|\, 2^{-\lfloor n/2 \rfloor - 1} \prod_{i = 1}^{\lfloor n/2 \rfloor} g(2^{-i}) \le \varepsilon \right\}.
\end{aligned}
\]
    
\end{restatable}

  A key observation is that in Lemma~\ref{lem:conv-rv}, the upper bound is achieved in the extreme case where the random variable $Z$ has only two atoms, at $a$ and $b$. This suggests an adversarial strategy for lower bounds: at each time step, choose the distribution of contrastive examples to be supported on two atoms, one at the target concept and the other at an endpoint of the current version space. With this key idea, we obtain a lower bound on the expected sample complexity of active learning for thresholds with respect to $\CEprob$.


\begin{restatable}{theorem}{ActiveProbThreshLower}
     \label{thm:active-prob-thresh-lower}
Let $f: \reals^{\geq 0} \to \reals^{\geq 0}$ be a non-decreasing function. Denote $g(x) := \min \!\left(\frac{f(x)}{x}, 1\right)$. Then 
\[
\begin{aligned}
    & \mathcal{M}_{\expc, \act} \left[\thresh, \CEprob\right](\varepsilon)  \geq\\
& \quad  
 \min \left\{ n \ge 0 \,\middle|\, 2^{-(2n+1)} \prod_{i = 1}^{n} g(4^{-i}) \le \varepsilon \right\}.
\end{aligned}
\]
\end{restatable}


The above theorem also implies a lower bound for the active contrastive sample complexity.

\begin{restatable}{corollary}{ActiveProbThreshPrLower}
    \label{cor:active-prob-thresh-pr-lower}
    Let $f: \reals^{\geq 0} \to \reals^{\geq 0}$ be a non-decreasing function. Denote $g(x) := \min \!\left(\frac{f(x)}{x}, 1\right)$. Then 
    \[
    \begin{aligned}
        & \mathcal{M}_{\Pr, \act} \left[\thresh, \CEprob\right](\varepsilon, \delta) \geq \\
    & \quad  
    \min \left(\min \left\{ n \ge 0 \,\middle|\, \prod_{i = 1}^{n} g(4^{-i}) \le \delta \right\},  \frac{\log_2 (1 / \varepsilon) - 1}{2} \right)\,.
    \end{aligned}
    \]
\end{restatable}



\subsection{Expected Sample Requirement for Accuracy $\varepsilon$} \label{sec:active-prob-acc-req}

In the previous subsection, we observed a significant drop-off in the active expected contrastive sample complexity from the deterministic to the probabilistic AMDM, even for threshold functions. In this subsection, instead of bounding the expected error, we focus on the expected number of samples a learner requires in order to guarantee error at most $\varepsilon$---a quantity, which we call the \emph{expected sample requirement for accuracy $\varepsilon$}. Interestingly, for this quantity in contrast with the previous subsection, we are able to derive bounds for both thresholds and half-spaces, and these bounds are similar to those we obtained in Section~\ref{sec:active-noisy} under the deterministic model, even in the probabilistic setting.



\begin{definition}
    Let $\cC$ be a concept class over $\cX$, let $\CE$ refer to a mechanism for selecting contrastive examples and let $\varepsilon \in (0, 1]$. Let $\mathcal{A}$ be an active learning algorithm. For every $\ell \in \cC$, let $N_{\cA, \ell,\varepsilon}$ be the random variable representing the minimal $m \in \mathbb{N}$, such that $\err(C_{\mathcal{A}, \ell, m}, \ell) \leq \varepsilon$. Here $C_{\mathcal{A}, \ell, m}$ is the final hypothesis output by $\mathcal{A}$ after interacting with the two-step oracle for $m$ steps. Then define $\mathcal{N}(\cC, \CE, \cA,\varepsilon) := \sup_{\ell \in \cC} \mathbb{E}[ N_{\cA, \ell, \varepsilon}]$ to be the \emph{expected sample requirement} of $\cA$ for achieving accuracy $\varepsilon$ on $\cC$ with respect to $\CE$. Also, define
    $$
    \mathcal{N}_{\act}(\cC, \CE, \varepsilon) = \inf_{\text{active algorithm }\cA} \mathcal{N}(\cC, \CE, \cA,\varepsilon)\,.
    $$
\end{definition}

We begin by analyzing threshold functions. Learning proceeds in several sub-phases. In each sub-phase, suppose the current version space is an interval of size $r$. The learner repeatedly queries both endpoints of the interval until their contrastive examples are less than $4f(r)$ apart. We then show that each sub-phase takes constant time in expectation. The key idea is to repeat a procedure until a benchmark is met; we refer to this procedure as \emph{verification}.



\begin{restatable}{theorem}{ActiveProbThreshAccReq}\label{thm:active-prob-thresh-acc-req}
    Let $f: \reals^{\geq 0} \to \reals^{\geq 0}$ be a non-decreasing
function. Then, 
    for every $\varepsilon \in (0, 1]$, 
    $$\mathcal{N}_{\act} \left(\thresh, \CEprob, \varepsilon\right) \leq  4 \STOP{4f}{1}{2\varepsilon}.$$
    
    
\end{restatable}

The above theorem also immediately implies an upper bound for the active contrastive sample complexity.

\begin{restatable}{corollary}{activeProbThreshPrUpper} \label{cor:active-prob-thresh-pr-upper}
    For any non-decreasing function $f:\reals^{\geq 0} \rightarrow \reals^{\geq 0}$, we have
    \[
    \begin{aligned}
         \mathcal{M}_{\Pr, \act}\left[\thresh, \CEprob\right](\varepsilon, \delta) \\ \leq 8 \STOP{4f}{1}{2\varepsilon} + 16 \ln \left(\frac{1}{\delta} \right).
    \end{aligned}  
    \]
\end{restatable}

    We now proceed to analyzing half-spaces. The core idea is similar to the threshold case: learning is again done in several sub-phases, and in each sub-phase the learner performs a verification process. The main difference is that verification is more intricate for half-spaces. Lemma~\ref{lem:HS-prob-helper} explicitly describes this verification procedure.

\begin{restatable}{lemma}{HSProbHelper}
    \label{lem:HS-prob-helper}
    Suppose $\cX = \mathcal{B}(\vec 0_k, \tfrac{1}{2})$ and let $\ell \in \HS$, $x \in \cX$. Let $f:\reals^{\geq 0} \rightarrow \reals^{\geq 0}$ be a non-decreasing function such that $f(a) \leq \tfrac{a}{4}$. Denote $r = \dist(x, \ell)$. Then there is a learner that, in expectation, with 4 queries from $\CE^{d, f}_{\mathrm{prob}}$ finds a $z^*$ such that (i) $\dist(z^*, \ell) \leq r$, and (ii) with probability at least $1/2$ we have $\dist(z^*, \ell) \leq 3 f(2r)$.
\end{restatable} 

\begin{restatable}{theorem}{ActiveProbUpperHS} \label{thm:active-prob-upper-HS}
    Let $f:\reals^{\geq 0} \rightarrow \reals^{\geq 0}$ be a non-decreasing function such that $f(x) \leq \tfrac{x^{1 + c}}{4}$. Then, for every $\varepsilon \in (0, 1]$, 
    $$\mathcal{N}_{\act} \left(\HS, \CEprob, \varepsilon \right) \leq 8 \STOP{\tilde f}{\frac{1}{2}}{\varepsilon'} + 8,$$
    where $\tilde f(x) = 3 f(2x)$ and $\varepsilon' = \left(\frac{\varepsilon}{2^{k - 1}}\right)^{\frac{1}{c}}$. 
\end{restatable}

\begin{remark}
    For half-spaces in contrast with threshold functions, the learner would not be able to verify when the error of its hypothesis is less than $\varepsilon$.
\end{remark}

The above theorem also implies an upper bound for the active contrastive sample complexity.

\begin{restatable}{corollary}{activeProbHSPrUpper} \label{cor:active-prob-HS-pr-upper}
    Let $f:\reals^{\geq 0} \rightarrow \reals^{\geq 0}$ be a non-decreasing function such that $f(x) \leq \tfrac{x^{1 + c}}{4}$. Then 
    \[
    \begin{aligned}
        & \mathcal{M}_{\Pr, \act}\left[\HS, \CEprob\right](\varepsilon, \delta) \\
        & \quad \leq
           32 \STOP{\tilde f}{\frac{1}{2}}{\varepsilon'} + 136 \log_2 \frac{3}{\delta}\,,
    \end{aligned}
    \]
   where $\tilde f(x) = 3 f(2x)$ and $\varepsilon' = \left(\frac{\varepsilon}{2^{k - 1}}\right)^{\frac{1}{c}}$.
\end{restatable}


An interesting phenomenon can be observed in Table~\ref{tab:active}. 
For the noise functions $f$ introduced, we observe a drop-off in the active \emph{expected} sample complexity when comparing $\CEprob$ to $\CEnoisy$, even for threshold functions. 
Nevertheless, for both $\cC \in \{\thresh,\HS\}$, the expected sample requirement for accuracy $\varepsilon$,
$\mathcal{N}_{\act}(\cC,\CEprob,\varepsilon)$, remains of the same order as
\[
\mathcal{M}_{\Pr,\act}\!\left[\cC,\CEnoisy\right](\varepsilon,1)
=
\mathcal{M}_{\exp,\act}\!\left[\cC,\CEnoisy\right](\varepsilon).
\]
Intuitively, this suggests that while with a small number of examples the learner is likely to achieve small error, there remains a non-negligible probability that the error is very large; equivalently, for a fixed sample size, the error distribution is highly skewed. 


\subsection{Passive Expected Contrastive Sample Complexity} \label{sec:passive-prob}
Finally, we analyze the passive expected contrastive sample complexity of one-dimensional threshold functions and homogeneous half-spaces with respect to the probabilistic AMDM.

First, consider learning of  threshold functions. We use Lemma~\ref{lem:highprob-to-expect} (stated below) to convert the high-probability guarantees from Theorem~\ref{thm:passive-noisy-min-thresh} for passive contrastive learning under $\CEnoisy$ into bounds in expectation. Next, observe that the algorithm in Theorem~\ref{thm:passive-noisy-min-thresh} relies on only a single contrastive example. Therefore, it achieves the same expected error under the probabilistic AMDM as under the deterministic one.

\begin{lemma} \label{lem:highprob-to-expect}
    Let $Z$ be a random variable such that $Z \in [0, \gamma]$ and $\Pr[ Z  > \varepsilon] \leq \delta$. Then
    $
     \mathbb{E}[Z] < \gamma \delta + \varepsilon (1 - \delta).
    $
\end{lemma}

\begin{restatable}{theorem}{PassiveProbThresh}
    \label{thm:passive-prob-exp-thresh}
    Let $f: \reals^{\geq 0} \to \reals^{\geq 0}$ be non-decreasing and invertible , 
    $\delta \in (0, 1]$, and $\varepsilon < \min\!\left( \frac{f(\tfrac{1}{4})}{4}, \tfrac{1}{8}\right)$. Then 
    \small{
    \[
    \begin{aligned}
        \mathcal{M}_{\exp, \passive}\left[\thresh, \CEprob\right](\varepsilon, \delta) 
        & \in O\!\left(\frac{\log(1/\varepsilon)}{\min(f^{-1}(\varepsilon / 2), \varepsilon)}\right), \\
        \mathcal{M}_{\exp, \passive}\left[\thresh, \CEprob\right](\varepsilon, \delta) 
        &  \in \Omega \!\left(\frac{1}{\min(f^{-1}(4\varepsilon), \varepsilon)}\right).
    \end{aligned}  
    \]}
\end{restatable}

For homogeneous half-spaces, we first show that the error of the same algorithm introduced in Theorem~\ref{thm:passive-noisy-min-HS} for passive contrastive learning with respect to $\CEnoisy$ can be bounded by a quantity that depends only on the point $x_{i^*}$ closest to the boundary $\ell$ and its contrastive example. This allows us to bound the algorithm's expected error, and in particular implies that the algorithm attains the same expected error under the probabilistic AMDM as under the deterministic one. Finally, we use Lemma~\ref{lem:highprob-to-expect} to convert the high-probability analysis from Theorem~\ref{thm:passive-noisy-min-HS} into bounds in expectation.

\begin{restatable}{theorem}{PassiveProbHS} \label{thm:passive-prob-HS}
    Let $\HHS = \{ \mathbbm 1\{ \langle \omega, \vec x \rangle \geq 0\} \mid \omega \in \reals^k \}$ for $\cX = \mathcal{B}(\vec 0_k, \frac{1}{2})$ be the class of homogeneous half-spaces. Let $f: \reals^{\geq 0} \to \reals^{\geq 0}$ be such that $f(x) \leq x^2$ and $g(x) := \frac{f(x)}{x}$ is non-decreasing and invertible. Then 
    \[
        \mathcal{M}_{\exp, \passive}\left[\HHS, \CEprob\right](\varepsilon)
        \;\in\; O \!\left(\frac{\log(1/\varepsilon)}{g^{-1}\left(\frac{\pi \varepsilon}{2} \right)}\right).
    \]
\end{restatable}

\section{Conclusion}

We studied learning from contrastive examples under the minimum distance paradigm, moving beyond the idealized setting in which the learner always receives an exact closest contrastive point. We introduced and analyzed two perturbed variants of this model, namely the deterministic approximate minimum distance mechanism and its probabilistic counterpart. While learning under the exact minimum distance model can be trivial for the concept classes we study, these perturbations render the problem non-trivial and lead to a richer sample-complexity landscape. 

Despite the added difficulty, both noisy mechanisms still yield substantial improvements over membership query learning. For thresholds, homogeneous half-spaces, and general half-spaces, we established upper and lower bounds showing that contrastive feedback can significantly reduce the number of queries required to learn an accurate classifier.

Beyond these quantitative bounds, our analysis also reveals a qualitative separation between two performance notions in the probabilistic model. Specifically, we observed that the expected number of samples required to \emph{guarantee} error at most $\varepsilon$ can differ markedly from the number of samples sufficient to achieve \emph{expected} error at most $\varepsilon$. This indicates that under probabilistic perturbations, the error after a fixed number of queries can be highly skewed---typically small, but occasionally very large.

\bibliography{ref}
\bibliographystyle{icml2026}
\newpage

\appendix

\section{Active Contrastive Complexity Bounds for Specific Classes}
\label{apdx:different-f}


In this section, we apply the results obtained in Sections~\ref{sec:active-noisy}, \ref{sec:active-prob-expected}, and \ref{sec:active-prob-acc-req} for active contrastive learning with respect to $\CEnoisy$ and $\CEprob$.
We instantiate our bounds for two representative families of noise functions:
(i) $\fpol$ satisfying $\frac{r^{1+c'}}{4} \le \fpol(r) \le \frac{r^{1+c}}{4}$, where $c,c' > 0$, and
(ii) $\fexp$ defined by $\fexp(r) := \frac{e^{-\tfrac{1}{r}}}{4}$.
A summary of these results appears in Table~\ref{tab:active}.

\begin{example} \label{ex:active-noisy-summary-thresh}
    The active contrastive sample complexity of thresholds with respect to the deterministic AMDM is
\[
\begin{aligned}
    \mathcal{M}_{\Pr,\act}\!\left[\thresh,\CEnoisyf{\fpol}\right](\varepsilon,\delta)
    &= \Theta\!\left(\log \log \frac{1}{\varepsilon}\right), \\
    \mathcal{M}_{\Pr,\act}\!\left[\thresh,\CEnoisyf{\fexp}\right](\varepsilon,\delta)
    &= \Theta\!\left(\log^*\frac{1}{\varepsilon}\right).
\end{aligned}
\]
\end{example}
\begin{proof}
Let $\tilde \fpol(r) := \fpol \circ (\fpol + \mathbbm I)^{-1}\!\left(\tfrac{r}{2}\right)$ ($\fpol$ applied to \eqref{eq:active-noisy-min-thresh}).  
Since $\tilde \fpol(r) \le \fpol(r) \le r^{(1 + c)}$, we obtain
\[
    \tilde \fpol^{(n)}(1) \leq \tilde \fpol^{(n - 1)} \left(\frac{1}{4}\right) \leq \left(\tfrac{1}{4}\right)^{(1+c)^{(n - 1)}}.
\]
Hence, using Theorem~\ref{thm:active-noisy-min-thresh} we derive
\[
\begin{aligned}
    \mathcal{M}_{\Pr,\act}\!\left[\thresh,\CEnoisyf{\fpol}\right](\varepsilon,\delta)
    &= \STOP{\tilde \fpol}{1}{2\varepsilon} \\
    &\le \log_4 \log_{(1+c)}(2\varepsilon) + 1
\end{aligned}
\]
On the other hand, since $\fpol \le \mathbbm I$, we have $(\fpol+\mathbbm I)^{-1} \ge \mathbbm I/2,$
and thus
\[
    \tilde \fpol(r) \ge \fpol(r/4)
\]
Specifically, for $r \le \tfrac{1}{4}$ we have $$\tilde \fpol(r) \ge \fpol(r/4) \geq \fpol(r^2) \geq r^{3 + 2c'}.$$
Therefore,
\[
    \tilde \fpol^{(n)}(1) 
    \ge \tilde \fpol^{(n-1)}\!\left(\fpol\!\left(\tfrac{1}{4}\right)\right) 
    \ge \left(\tfrac{1}{4}\right)^{ (3 + 2c')^{n}},
\]
and consequently, by Theorem~\ref{thm:active-noisy-min-thresh},
\[
\begin{aligned}
    \mathcal{M}_{\Pr,\act}\!\left[\thresh,\CEnoisyf{\fpol}\right](\varepsilon,\delta)
    &= \STOP{\tilde \fpol}{1}{2\varepsilon}\\
    &\ge \log_4 \log_{(3+2c')}(2\varepsilon).
\end{aligned}
\]

For $\fexp$, define $\tilde \fexp(r) := \fexp \circ (\fexp + \mathbbm I)^{-1}\!\left(\tfrac{r}{2}\right)$. 
Combining the fact that $\fexp(r/2) \le \tilde \fexp(r) \le \fexp(r)$ with Theorem~\ref{thm:active-noisy-min-thresh} completes the proof. 
\end{proof}


\begin{example} \label{ex:active-noisy-summary-HS}
The active contrastive sample complexity of half-spaces under the deterministic AMDM satisfies
\[
\begin{aligned}
    \mathcal{M}_{\Pr,\act}\!\left[\HS,\CEnoisyf{\fpol}\right](\varepsilon,\delta)
    &= O\!\left(\log \log \frac{1}{\varepsilon} + \log k\right),  \\
    \mathcal{M}_{\Pr,\act}\!\left[\HS,\CEnoisyf{\fpol}\right](\varepsilon,\delta)
    &= \Omega\!\left(\log \log \frac{1}{\varepsilon} - \log k\right),  \\
    \mathcal{M}_{\Pr,\act}\!\left[\HS,\CEnoisyf{\fexp}\right](\varepsilon,\delta)
    &= O\!\left(\log^*\frac{1}{\varepsilon} + \log^*k \right), \\
    \mathcal{M}_{\Pr,\act}\!\left[\HS,\CEnoisyf{\fexp}\right](\varepsilon,\delta)
    &= \Omega\!\left(\log^*\frac{1}{\varepsilon} - \log^*k \right).
\end{aligned}
\]
\end{example}

\begin{proof}
    For the class of half-spaces, all lower bounds follow by combining the corresponding lower bounds for thresholds with Theorem~\ref{thm:active-noisy-min-lower-HS}. For the upper bound with $\fpol$, we apply Theorem~\ref{thm:active-noisy-min-upper-HS} and use  $\fpol(r) \le r^{1+c}$ to obtain

\[
\begin{aligned}
    & \mathcal{M}_{\Pr,\act}\!\left[\HS,\CEnoisyf{\fpol}\right](\varepsilon,\delta) \\
    &\quad =  \STOP{\fpol}{\frac{1}{2}}{\left(\frac{\varepsilon}{2^k}\right)^{\frac{1}{c}}} \\
    &\quad \le \log_2 \log_{(1+c)}\!\left(\left(\frac{2^k}{\varepsilon}\right)^{1/c}\right) \\
    &\quad = \log_2\!\left(  \frac{ \log_{(1 + c)}(2) \left( k + \log_2\!\left(\frac{1}{\varepsilon} \right)\right)}{c}\right) \\
    &\quad = O\!\left(\log \log \frac{1}{\varepsilon} + \log k\right).
\end{aligned} 
\]
For $\fexp$, note that for all $r \le \frac{1}{2}$ we have $\fexp(r) \le r^{1 + c}$ for $c = 1$.
Thus, we can use Theorem~\ref{thm:active-noisy-min-upper-HS} and derive
\[
\begin{aligned}
    \mathcal{M}_{\Pr,\act}\!\left[\HS,\CEnoisyf{\fexp}\right](\varepsilon,\delta)
    &= \STOP{\fexp}{\frac{1}{2}}{\left(\frac{\varepsilon}{2^k}\right)} \\
    &= O \left( \log^* \frac{2^k}{\varepsilon} \right) \\
    & = O \left( \log^* \frac{1}{\varepsilon} + \log^*k \right).
\end{aligned} 
\]

\end{proof}

\begin{example} \label{ex:active-prob-summary-exp}
    The active expected contrastive sample complexity of thresholds with respect to the probabilistic AMDM satisfies
\[
\begin{aligned}
\mathcal{M}_{\exp,\act}\!\left[\thresh,\CEprobf{\fpol}\right](\varepsilon)
& =
\Theta\!\left(\sqrt{\log \frac{1}{\varepsilon}}\right) \\
\mathcal{M}_{\exp,\act}\!\left[\thresh,\CEprobf{\fexp}\right](\varepsilon)
& =
\Theta\!\left(\log \log \frac{1}{\varepsilon}\right).
\end{aligned}
\]
\end{example}

\begin{proof} For $\fpol$, denote $g(x) := \fpol(x) / x$. Then
    $$
        \begin{aligned}
        2^{-\lfloor n/2 \rfloor - 1} \prod_{i = 1}^{\lfloor n/2 \rfloor} g(2^{-i})
        & \leq
        2^{-\lfloor n/2 \rfloor - 1} \prod_{i = 1}^{\lfloor n/2 \rfloor} 2^{-(ci + 2)} \\
        & =
        2^{-\Omega(n^2)},
        \end{aligned}
        $$
        and
        \begin{equation} \label{eq:active-exp-all-res}
            2^{-(2n + 1)} \prod_{i = 1}^{n} g(2^{-i})
            \geq
            2^{-(2n + 1)} \prod_{i = 1}^{n} 2^{-(c'i + 2)}
            =
            2^{-O(n^2)}.
        \end{equation}
        Therefore, by plugging in Theorem~\ref{thm:active-prob-thresh-upper} and Theorem~\ref{thm:active-prob-thresh-lower} we get
        $$\mathcal{M}_{\exp, \act}\left[\thresh, \CEprobf{\fpol}\right](\varepsilon, \delta)
        = \Theta \left(\sqrt{\log \frac{1}{\varepsilon}}\right)$$

        Also, for $g(x) = \fexp(x) / x$ we have
        $$
        \begin{aligned}
            2^{-\lfloor n/2 \rfloor - 1} \prod_{i = 1}^{\lfloor n/2 \rfloor} g(2^{-i})
        & = 2^{-\lfloor n/2 \rfloor - 1} \prod_{i = 1}^{\lfloor n/2 \rfloor} 2^{i + 2} \exp(-2^i)\\
        & = \exp \left(-2^{\Omega(n)} \right)
        \end{aligned}
        $$
        Similarly
        \begin{equation*} 
            2^{-(2n + 1)} \prod_{i = 1}^{n} g(2^{-i})
        =
        \exp \left(-2^{O(n)} \right)
        \end{equation*}
        Therefore, $$\mathcal{M}_{\exp, \act}\left[\thresh, \CEprobf{\fexp} \right](\varepsilon, \delta)
        = \Theta \left(\log \log \frac{1}{\varepsilon}\right).
        $$        

\end{proof}

\begin{example}
    The expected sample requirement for accuracy $\varepsilon$ of thresholds with respect to the probabilistic AMDM is bounded from above by
    $$
    \begin{aligned}
        \mathcal{N}_{\act}[\thresh, \CEprobf{\fpol}](\varepsilon, \delta) & = O \left( \log \log \frac{1}{\varepsilon} \right) \\
         \mathcal{N}_{\act}[\thresh, \CEprobf{\fexp}](\varepsilon, \delta) & = O \left( \log^* \frac{1}{\varepsilon} \right)
    \end{aligned}
    $$
    Moreover, the active contrastive sample complexity of thresholds with respect to the probabilistic AMDM satisfies
    
    $$
     \begin{aligned}
        \mathcal{N}_{\act}[\thresh, \CEprobf{\fpol}](\varepsilon, \delta) & = O \left( \log \log \frac{1}{\varepsilon} +\log \frac{1}{\delta}\right) \\
        \mathcal{M}_{\Pr, \act}[\thresh, \CEprobf{\fpol}](\varepsilon, \delta) & = \Omega \left( \log \log \frac{1}{\varepsilon} + \sqrt{\log \frac{1}{\delta}}\right) \\
         \mathcal{N}_{\act}[\thresh, \CEprobf{\fexp}](\varepsilon, \delta) & = O \left( \log^* \frac{1}{\varepsilon} + \log \frac{1}{\delta} \right) \\
          \mathcal{M}_{\Pr, \act}[\thresh, \CEprobf{\fexp}](\varepsilon, \delta) & = \Omega \left( \log^* \frac{1}{\varepsilon} + \log \log \frac{1}{\delta} \right)
    \end{aligned}
    $$
    
\end{example}
\begin{proof}
    With arguments similar to those used in Example~\ref{ex:active-noisy-summary-thresh}, it is straightforward to see that
    $$\STOP{4\fpol}{1}{2\varepsilon} = O(\log \log \tfrac{1}{\varepsilon})$$
    and  
    $$ 
    \STOP{4\fexp}{1}{2\varepsilon} = O(\log^* \tfrac{1}{\varepsilon}).
    $$ Plugging these into Theorem~\ref{thm:active-prob-thresh-upper} and Corollary~\ref{cor:active-prob-thresh-pr-upper} yields the upper bounds.

    Since learning is harder in the probabilistic AMDM  than in the deterministic AMDM, we can use Example~\ref{ex:active-noisy-summary-thresh} to derive
\begin{equation} \label{eq:active-prob-acc-req-all-1}
    \mathcal{M}_{\Pr, \act}[\thresh, \CEprobf{\fpol}](\varepsilon, \delta) = \Omega \left( \log \log \frac{1}{\varepsilon}\right).
\end{equation}
Denote $g(x) := \fpol(x) / x$. Similar to the argument in \eqref{eq:active-exp-all-res} of Example~\ref{ex:active-prob-summary-exp}, we have
$$
\min \left\{ n \ge 0 \,\middle|\, \prod_{i = 1}^{n} g(4^{-i}) \le \delta \right\} = \Omega\left(\sqrt{\log \frac{1}{\delta}}\right).
$$ 
Plugging this into Corollary~\ref{cor:active-prob-thresh-pr-lower} gives
$$
\begin{aligned}
& \mathcal{M}_{\Pr, \act}[\thresh, \CEprobf{\fpol}](\varepsilon, \delta) \\
& \quad \geq \Omega \left(\min\left(\sqrt{\log \frac{1}{\delta}}, \log \frac{1}{\varepsilon}\right)\right).
\end{aligned}
$$
Combining this with \eqref{eq:active-prob-acc-req-all-1} yields our lower bound for $\mathcal{M}_{\Pr, \act}[\thresh, \CEprobf{\fpol}](\varepsilon, \delta)$. The proof for $\fexp$ is identical.
\end{proof}


    
\begin{example}
    The expected sample requirement for accuracy $\varepsilon$ of half-spaces with respect to the probabilistic AMDM is bounded from above by
    $$
    \begin{aligned}
        \mathcal{N}_{\act}[\HS, \CEprobf{\fpol}](\varepsilon, \delta) & = O \left( \log \log \frac{1}{\varepsilon} + \log k \right) \\
         \mathcal{N}_{\act}[\HS, \CEprobf{\fexp}](\varepsilon, \delta) & = O \left( \log^* \frac{1}{\varepsilon} + \log^* k \right) 
    \end{aligned}
    $$
    Moreover, the active contrastive sample complexity of half-spaces with respect to the probabilistic AMDM satisfies
    \small{
    $$
     \begin{aligned}
        \mathcal{M}_{\Pr, \act}[\HS, \CEprobf{\fpol}](\varepsilon, \delta) & = O \left( \log \log \frac{1}{\varepsilon} + \log \frac{1}{\delta} + \log k\right) \\
        \mathcal{M}_{\Pr, \act}[\HS, \CEprobf{\fpol}](\varepsilon, \delta) & = \Omega \left( \log \log \frac{1}{\varepsilon} - \log k \right) \\
         \mathcal{M}_{\Pr, \act}[\HS, \CEprobf{\fexp}](\varepsilon, \delta) & = O \left( \log^* \frac{1}{\varepsilon} + \log \frac{1}{\delta} + \log^* k \right) \\
          \mathcal{M}_{\Pr, \act}[\HS, \CEprobf{\fexp}](\varepsilon, \delta) & = \Omega \left( \log^* \frac{1}{\varepsilon} - \log^* k \right)
    \end{aligned}
    $$}
\end{example}
\begin{proof}
    Denote $\tilde \fpol = 3 \fpol(2x)$. Again, with arguments similar to those used in Example~\ref{ex:active-noisy-summary-thresh}, it is straightforward to see that
    $$
    \STOP{\tilde \fpol}{\frac{1}{2}}{\left(\frac{\varepsilon}{2^{k - 1}}\right)^{\frac{1}{c}}} = O(\log \log \tfrac{1}{\varepsilon} + \log k)
    $$  
    Similarly denote $\tilde \fexp = 3 \fexp(2x)$, then 
    $$\STOP{\tilde \fexp }{\frac{1}{2}}{\left(\frac{\varepsilon}{2^{k - 1}}\right)^{\frac{1}{c}}} = O(\log^* \tfrac{1}{\varepsilon} + \log^* k)$$ 
    Plugging these into Theorem~\ref{thm:active-prob-upper-HS} and Corollary~\ref{cor:active-prob-HS-pr-upper} yields the upper bounds. 
The lower bounds follow immediately from Example~\ref{ex:active-noisy-summary-HS}.   
\end{proof}


\section{Missing Proofs from Section~\ref{sec:active-noisy}} \label{apdx:active-noisy}
\ActiveNoisyThresh*
\begin{proof}
Denote $m = \STOP{\tilde f}{1}{2\varepsilon}$.

\textbf{Upper bound:}  
For a time $t$, suppose the valid positions of $\theta$ span an interval $I_t = [a, b] \subseteq [0,1]$. 
We first show that there exists an $x_t$ such that for all possible valid contrastive examples $x'_t$ for $x_t$, the valid positions for $\theta$ will span an interval $I_{t + 1}$ such that  
\[
|I_{t + 1}| \leq \tilde f\!\left(|I_t|\right).
\]  
Proving this completes the argument: after $m$ queries, every valid position will span an interval $I_m$ of length at most $2\varepsilon$. Consequently, the threshold corresponding to the center of $I_m$ is guaranteed to have error less than $\varepsilon$ with probability 1.

To prove the claim, let $t \geq 1$ and choose $x_t = \tfrac{a+b}{2}$. Without loss of generality, suppose the contrastive example $x'_t$ received by the learner satisfies $x'_t > x_t$. Since the labels of $x'_t$ and $x_t$ are different, $\theta$ always lies between $x_t$ and $x'_t$. Therefore $|x'_t - \theta| + |x_t - \theta| = |x'_t - x_t|$. Moreover, $|x'_t - \theta| \leq f\big(|x_t - \theta|\big).$  
Thus,  
$f^{-1} \big(|x_t' - \theta|\big) + |x'_t - \theta| \leq |x_t - x'_t|$. It follows that  
$$|x'_t - \theta| \leq (f^{-1} + \mathbbm I)^{-1}\!\big(|x_t - x'_t|\big) = f\cdot (f + \mathbbm I)^{-1}\!\big(|x_t - x'_t|\big) $$
Hence, the refined interval is  
\[
I_{t + 1} = \Big[x'_t - f \circ (f+ \mathbbm I)^{-1}\!\big(|x_t - x'_t|\big), \; x'_t\Big] \cap I_t.
\]  
The size of $I_{t + 1}$ is maximized when $x'_t = b$, yielding  
\[
|I_{t + 1}| \leq f \circ (f+\mathbbm I)^{-1}\!\left(\tfrac{b-a}{2}\right) = \tilde f\!\left(|I_t|\right).
\]

\textbf{Lower bound:}  
At time $t$, suppose the valid positions of $\theta$ span an interval $I_t = [a, b] \subseteq [0,1]$.  
We first show that for every $x_t \in \cX$ there exists a valid contrastive example $x'_t$ for $x_t$ such that the valid positions for $\theta$ will span an interval $I_{t+1}$ satisfying  
\[
|I_{t+1}| \;\geq\; \tilde f\!\left(|I_t|\right).
\]  
Proving this completes the argument: after $m$ queries, the valid positions will span an interval $I_m$ of length greater than $2\varepsilon$. Consequently,  
\[
\Pr\!\left[ \mathbbm{1}\{x \leq a\} \neq \mathbbm{1}\{x \leq b\} \right] > 2\varepsilon,
\]  
which implies that every function will have error more than $\varepsilon$ with respect to some $ \ell \in \left\{\mathbbm{1}\{x \leq a\}, \mathbbm{1}\{x \leq b\} \right\}$. 

\medskip

To prove the claim, without loss of generality, suppose $x_t \leq (a+b)/2$.  
Assume the contrastive example is given by $x'_t = b$.  
By a reasoning similar to the upper bound argument, the refined interval is  
\[
I_{t+1} = \Big[b - f \circ (f+\mathbbm I)^{-1}\!\big(b - x_t\big), \; b\Big],
\]  
which yields  $|I_{t+1}| \;\geq\; \tilde f\!\left(|I_t|\right)$.
\end{proof}

\ErrDist*
\begin{proof} 
    Consider all $\vec x'$ with fixed distance from $\vec x^{\proj}$. For such $\vec x'$, the angle
$\angle(\vec x',\vec x,\vec x^{\proj})$ is maximized when $\vec x' - \vec x^{\proj}$ is orthogonal to $\vec x - \vec x^{\proj}$ (see Figure~\ref{fig:circ_max} for an illustration). 
Thus, $\angle(\vec x',\vec x,\vec x^{\proj})$ is always bounded from above by
\begin{equation} \label{eq:err-dist}
    \tan^{-1}\!\left(\frac{\|\vec x' - \vec x^{\proj}\|}{\|\vec x - \vec x^{\proj}\|}\right) 
    \leq \frac{\|\vec x' - \vec x^{\proj}\|}{\|\vec x - \vec x^{\proj}\|}.
\end{equation}

For $\vec x'$ with fixed distance from $\vec x$, the angle $\angle(\vec x',\vec x,\vec x^{\proj})$ is also maximized in the same situation. Therefore, it is also bounded from above by
\begin{equation*} 
    \cos^{-1}\!\left(\frac{\|\vec x - \vec x^{\proj}\|}{\|\vec x - \vec x'\|}\right).
\end{equation*}
This completes the proof of the first part.

Next, observe that the angles between $C^0$ and $C^*$ and between $C$ and $C^*$ both equal $\angle(\vec x', \vec x, \vec x^{\proj})$. Therefore, from \eqref{eq:err-dist} we derive
\[
    \err(C^0, C^*) \leq \frac{\|\vec x' - \vec x^{\proj}\|}{\pi \|\vec x - \vec x^{\proj}\|},
\]
which completes the third part of the lemma.

It remains to prove the second part of the lemma. Let $C'$ denote the half-space parallel to $C^*$ that passes through $\vec x'$ and satisfies $C'(\vec x) = C^*(\vec x)$ (see Figure~\ref{fig:circ_general} for an illustration of $C'$).
Note that the mass of the region between $C'$ and $C^*$ inside $\cX$ is at most $\|\vec x' - \vec x^{\proj}\|$. Therefore,
\begin{equation} \label{eq:err-dist-1}
\begin{aligned}
     \err(C', C^*) 
     & \leq \frac{\|\vec x' - \vec x^{\proj}\|}{\VOL(\cX)} 
     \leq 2^k \|\vec x' - \vec x^{\proj}\| \\
     & \leq \frac{2^{k-1} \|\vec x' - \vec x^{\proj}\|}{\|\vec x - \vec x^{\proj}\|}.
\end{aligned}
\end{equation}
Now let $B' = \mathcal{B}(\vec x', 1) \supseteq \cX$. Since $\err(C, C')$ is the measure of the region between $C$ and $C'$ inside $\cX$, normalized by $\VOL(\cX)$, it is at most the corresponding region inside $B'$, normalized by $\VOL(\cX)$. 
Moreover, the region between $C$ and $C'$ in $B'$ is simply the angle between $C$ and $C'$ (which is $\angle(\vec x', \vec x, \vec x^{\proj})$) multiplied by $\frac{\VOL(B')}{\pi}$. Hence, by \eqref{eq:err-dist} we have
\begin{equation} \label{eq:err-dist-2}
\begin{aligned}
     \err(C, C') 
     & \leq \frac{\|\vec x' - \vec x^{\proj}\|\,\VOL(B')}{\|\vec x - \vec x^{\proj}\|\,\VOL(\cX)\,\pi} \\
     & = \frac{2^k \|\vec x' - \vec x^{\proj}\|}{\|\vec x - \vec x^{\proj}\|\,\pi} 
     \leq \frac{2^{k-1} \|\vec x' - \vec x^{\proj}\|}{\|\vec x - \vec x^{\proj}\|}.
\end{aligned}  
\end{equation}
Combining \eqref{eq:err-dist-1} and \eqref{eq:err-dist-2} completes the proof.
\end{proof}

\ActiveProbHHS*
\begin{proof}
    Choose any point $\vec x$ with $\|\vec x\|=r$, where $r := \left(\pi\varepsilon\delta\right)^{1/c}$, and let $\vec x'$ be its contrastive example.
Then
\[
\frac{f(r)}{r} \le \frac{r^{1+c}}{r} = r^c = \pi\varepsilon\delta.
\]
Even under the probabilistic AMDM we have 
$\mathbb{E}\big[\|\vec x' - \vec x^{\proj}\|\big] \le f(r)$,
and hence by Markov's inequality,
\[
\Pr\!\left[\|\vec x' - \vec x^{\proj}\| \ge \frac{f(r)}{\delta}\right] \le \delta.
\]
Therefore, with probability at least $1-\delta$,
\[
\frac{\|\vec x' - \vec x^{\proj}\|}{\|\vec x - \vec x^{\proj}\|}
\;\le\;
\frac{f(r)/\delta}{r}
\;\le\;
\pi\varepsilon.
\]
Plugging this into Lemma~\ref{lem:err-dist} (iii) completes the proof. 
\end{proof}

\ActiveNoisyLowerHS*
\begin{proof}
    Set $\varepsilon' = 2^{k}\varepsilon$. For any $p \in \left[-\tfrac{1}{4}, \tfrac{1}{4}\right]$ and $\vec x \in \reals^k$, define  
$C_p(\vec x) := \mathbbm{1}\{\vec x[0] \leq p\}$.  
Define the subclass
\[
\cC' := \left\{ C_p \;\middle|\; p \in \left[-\tfrac{1}{4}, \tfrac{1}{4}\right] \right\},
\]
and let $\ell$ be a concept in $\cC'$.

Since the version space depends only on the $0$-th coordinate of the samples, the class $\cC'$ effectively reduces to threshold functions on the interval $\left[-\tfrac{1}{4}, \tfrac{1}{4}\right]$. By Theorem~\ref{thm:active-noisy-min-thresh}, there exists a deterministic AMDM oracle such that, regardless of the learner’s queries, after
\[
\mathcal{M}_{\Pr, \act}[\thresh, \CE^{d, f}_{\mathrm{noisy\text{-}min}}](\varepsilon', \delta)
\]
queries there exist $a, b \in \left[-\tfrac{1}{4}, \tfrac{1}{4}\right]$ with $b - a \geq \varepsilon'$ (since the length of $\left[-\tfrac{1}{4}, \tfrac{1}{4}\right]$ is $\tfrac{1}{2}$, the bounds scale accordingly by a factor of $\tfrac{1}{2}$) such that the subset $\{ C_p \mid p \in [a, b]\}$ remains a valid version space for $\ell$. 

Observe that the rectangle
\(
[a,b] \times \left[-\tfrac{1}{4}, \tfrac{1}{4}\right]^{k-1}
\)
is contained in the disagreement region between $C_a$ and $C_b$ over $\cX$. Therefore,
\[
\Pr\!\left[ C_a(\vec x) \neq C_b(\vec x) \right] 
> \frac{\varepsilon'}{2^{k - 1} \VOL(\cX)} 
\geq 2\varepsilon.
\]

Consequently, 
by an argument analogous to that of Theorem~\ref{thm:active-noisy-min-thresh}, 
every learning algorithm incurs error greater than $\varepsilon$ 
for at least one of the concepts in $\{C_a, C_b\}$.
\end{proof}

\section{Missing Proofs from Section~\ref{sec:passive-noisy}}

\PassiveNoisyThresh*
\begin{proof}
    \textbf{Upper bound:}  
    Without loss of generality, assume there exists a $x_i$ with $\ell(x_i) = 1$. Let $x^*$ denote the maximum such $x_i$. If $f(\varepsilon) \leq \varepsilon$, the learner outputs the threshold corresponding to the contrastive example for $x^*$. Otherwise, it outputs the threshold corresponding to $x^*$.  

    Define $\varepsilon' = \min(f^{-1}(\varepsilon), \varepsilon)$. For 
    \(
        m = \frac{\log(1/\delta)}{\varepsilon'},
    \)
     the probability that after $m$ random samples we have $x^* < \theta - \varepsilon'$ is at most
    \[
        (1 - \varepsilon')^m \;\leq\; e^{-\varepsilon' m} \;=\; \delta.
    \]
    By the definition of contrastive examples, this completes the proof of the upper bound.

    \textbf{Lower bound:}  
   Denote
    $$
    x^- = \min_{i: \ell(x_i) = 0} x_i, \quad x^+ = \max_{i: \ell(x_i) = 1} x_i\,.
    $$
Define 
\(
    r := \tfrac{x^- - x^+}{2}.
\)
Suppose the contrastive example for every negatively labeled point is 
\(
    \tilde x^- := \tfrac{x^- + x^+ - f(r)}{2},
\)
and for every positively labeled point is 
\(
    \tilde x^+ := \tfrac{x^- + x^+ + f(r)}{2}.
\)

We first show that with this choice of contrastive examples all points in the interval 
\(
    I \;=\; \Big[\max(x^+, \tilde x^-), \; \min(x^-, \tilde x^+)\Big]
\)
are valid locations for $\theta$. Denote $r' = |I|$. Notice that if $\frac{f(r)}{2} \leq r$, then $r' = 2r$; otherwise, $r' = f(r)$. Therefore, \(
    r' = \min\big(f(r), 2r\big).
\)

Note that, by definition, if $\tilde x^+$ is a valid contrastive example for $x^+$, it must also be a valid contrastive example for \emph{every}\/ positively labeled point. Therefore, the contrastive examples for negative points restrict the location of valid $\theta$ to
\[
\begin{aligned}
    I^+ := \left[\max \!\left( x^+, \tilde x^+ - f\left(\tilde x^+ - x^+ \right)\right), \; \tilde x^+\right] 
\end{aligned}
\]
Observe that
$$
\tilde x^+ - f\left(\tilde x^+ - x^+ \right) =  \tilde x^+ - f\left(r + \frac{f(r)}{2}\right) \leq \tilde x^+ - f(r) = \tilde x^-
$$
This implies $\Big[\max(x^+, \tilde x^-), \; \tilde x^+\Big] \subseteq I^+$.

Similarly, the contrastive examples for negative points restrict the location of valid $\theta$ to a superset of
\[
    \Big[\tilde x^-, \; \min(x^-, \tilde x^+)\Big].
\]
This completes the proof of the claim.

Consequently, 
by an argument analogous to that of Theorem~\ref{thm:active-noisy-min-thresh}, 
every learning algorithm incurs error at least $\frac{r'}{2}$.


    Note if $r \geq \varepsilon'$ for $\varepsilon' := \min\!\left(f^{-1}(2\varepsilon), \varepsilon \right)$ we would have $r' \geq 2\varepsilon$. Thus it remains to show that after
    \(
        m \leq \frac{\log(1 / \delta)}{4 \varepsilon'}
    \)
    queries, the probability that $r > \varepsilon'$ would be at least $\delta$. Without loss of generality, suppose $\theta > 1/2$. Note that the probability of $r > \varepsilon'$ is at least the probability of $|x^+ - \theta| > 2\varepsilon'$, which for $\varepsilon' < 1/4$ is bounded from below by
    \[
        (1 - 2\varepsilon')^m \;\geq\; e^{-4 \varepsilon' m} \;\geq\; \delta.
    \]
    This completes the proof.
\end{proof}

\PassiveNoisyHS*
\begin{proof}
   Let 
   $$
   \vec x^* = \argmin_{x_i} \|x_i - x'_i\|
   $$
   be the primary example with minimal distance to its contrastive example, and $\tilde{\vec x}^*$ be the contrastive example for $\vec x^*$. Our algorithm simply outputs the homogeneous half space $C$ perpendicular to $\tilde{\vec x}^* - \vec x^*$.  
    
Set $\varepsilon' := \tfrac{g^{-1}(\pi \varepsilon)}{2}$. Note that for any $a \in (0, 0.5)$, the points at distance $a$ from $\ell$ form a $(k-1)$-dimensional sphere whose radius decreases as $a$ increases. Thus, the probability of a point sampled from $\mathcal{U}[\cX]$ being at most $\varepsilon'$ away from $\ell$ is at most $2 \varepsilon'$.  

Therefore, for $m = \tfrac{\ln(1 / \delta)}{2 \varepsilon'}$, the probability of $\min_{x \in S} \dist(x, \ell) > \varepsilon'$ is at most  
$$(1 - 2\varepsilon')^m \leq e^{-2\varepsilon' m} \leq \delta.$$  

Thus, with probability at least $1 - \delta$, we have $\|\tilde{\vec x}^* - \vec x^*\| \leq \varepsilon' + f(\varepsilon') \leq 2 \varepsilon'$.
Due to Lemma~\ref{lem:err-dist}~(iii), with probability at least $1 - \delta$,
$$
\err(C, \ell) \leq \frac{f(2 \varepsilon')}{ 2 \pi \varepsilon'} = \varepsilon.
$$
\end{proof}

\section{Missing Proofs from Section~\ref{sec:active-prob-expected}}

\begin{restatable}{lemma}{multconvincreas}\label{lem:mult-conv-increas}
Let $f, g:\reals \rightarrow \reals$ be convex and non-decreasing functions such that $f, g \geq 0$. Then $f \cdot g$ is also convex.
\end{restatable}
\begin{proof}
    Note that for all $x, y \in \reals$ and $a \in [0, 1]$ we have
    $$
    \begin{aligned}
    & f.g\left(ax + (1-a)y \right) \\
    &\stackrel{(i)}{\leq} \left(a f(x) + (1-a) f(y)\right)\cdot \left(a g(x) + (1-a) g(y)\right) \\
    & = a^2 f(x) g(x) + (1 - a)^2 f(y) g(y) \\
    & \quad \quad + a (1 - a)  (f(x) g(y) + f(y).g(x)) \\
    & = a f(x)g(x) + (1 - a) f(y)g(y) + a (1 - a) \big(f(x) g(y) \\
    & \quad \quad  + f(y)g(x) - f(x)g(x) - f(y)g(y)\big) \\
    & = a f(x) g(x) + (1 - a) f(y) g(y) \\
    & \quad \quad - a (1 - a) \big (g(y) - g(x)\big) \big(f(y) - f(x) \big) \\
    & \stackrel{(ii)}{\leq} a f.g(x) + (1 - a) f.g(y)
    \end{aligned}
    $$
    Here (i) holds due to the convexity of $f$ and $g$ and the fact that $f, g \geq 0$, and (ii) is due to the non-decreasing property of $f$ and $g$.
\end{proof}

\convrv*
\begin{proof}
    Note that for all $x \in [a, b]$, we can re write $x$ as
    $$
    x = \frac{x - a}{b - a} b + \frac{b - x}{b - a} a.
    $$
    Due to the convexity of $f$, this implies
    $$
    f(x) = f\left(\frac{x - a}{b - a} b + \frac{b - x}{b - a} a\right)  \leq \frac{x - a}{b - a} f(b) + \frac{b - x}{b - a} f(a)
    $$
    Therefore,
    $$
    \begin{aligned}
        \mathbb{E} [f(Z)] & \leq \mathbb{E} \left[ \frac{Z - a}{b - a} f(b) + \frac{b - Z}{b - a} f(a) \right]\\
        & = \frac{\mathbb{E}[Z] - a}{b - a} f(b) + \frac{b - \mathbb{E}[Z] }{b - a} f(a) \\
    \end{aligned}
    $$
    This completes the proof.
\end{proof}
\ActiveProbThreshUpper*
\begin{proof}
Before proceeding to the proof, we first show that for any $x, y \geq 0$. $f(x + y) \geq f(x) + f(y)$. Note that since $g$ is non-decreasing we have $\frac{f(x + y)}{x + y} \geq \frac{f(x)}{x}$. Thus, 
\begin{equation} \label{eq:sup-add-1}
    \frac{x \cdot f(x + y)}{x + y} \geq f(x)
\end{equation}
Similarly,
\begin{equation} \label{eq:sup-add-2}
    \frac{y \cdot f(x + y)}{x + y} \geq f(y)
\end{equation}
By adding \eqref{eq:sup-add-1} to \eqref{eq:sup-add-2} we derive 
\begin{equation} \label{eq:sup-add-3}
    f(x + y) \geq f(x) + f(y)
\end{equation}

Let $m$ be an arbitrary number. Starting with $I_0 = [0,1]$, in each subroutine $t \in [m']$ (where $m' = \lfloor m/2 \rfloor$), suppose $\theta \in I_t = [a_t, b_t]$. The learner queries $x_{2t-1} = \frac{a_t + b_t}{2}$. Without loss of generality, suppose the label of $x_{2t-1}$ is $1$, and the learner receives a contrastive example $x'_{2t-1}$. Then the learner queries $x_{2t} = b_t$ and receives $x'_{2t}$.

Therefore, given a fixed $x'_{2t - 1}$ and $x'_{2t}$ we have
\[
\begin{aligned}
    \mathbb{E}[x'_{2t-1} - x'_{2t}] & \leq f(\theta - x_{2t-1}) + f(x_{2t} - \theta) \\
    & \stackrel{\eqref{eq:sup-add-3}}{\le} f(x_{2t} - x_{2t - 1}) = f\!\left(\frac{r_t}{2}\right),
\end{aligned}
\]
where $r_t = b_t - a_t$. Hence, after this round the next interval is 
\[
I_{t+1} = [\min(x'_{2t}, x_{2t-1}), \max(x'_{2t-1}, x_{2t})],
\]
so that $r_{t+1} \in [0, r_t/2]$ and
\begin{equation}\label{eq:t-to-t+1}
\mathbb{E}[r_{t+1} \mid r_t] \le f \left(\frac{r_t}{2}\right)
\end{equation}
If the learner outputs the midpoint of $I_{m'}$, its expected error is $\frac{1}{2}\mathbb{E}[r_{m'}]$.

We now prove by induction that for all $t \leq m'$ we have
\[
\mathbb{E}[r_{m'}] \le \mathbb{E}\!\left[ \frac{r_{m'-t}}{2^t} \cdot \prod_{i=1}^{t} g\!\left(\frac{r_{m'-t}}{2^i}\right) \right].
\]
\textit{Base case.} For $t=1$:
\[
\begin{aligned}
    \mathbb{E}[r_{m'}] 
& = \mathbb{E}\!\left[\mathbb{E}[r_{m'} \mid r_{m'-1}]\right] 
\stackrel{\eqref{eq:t-to-t+1}}{\le} \mathbb{E}\!\left[f \!\left(\frac{r_{m'-1}}{2}\right)\right] \\
 & = 
\mathbb{E}\!\left[\frac{r_{m'-1}}{2} \cdot g \!\left(\frac{r_{m'-1}}{2}\right)\right].
\end{aligned}
\]
\textit{Inductive step.} Assume the claim holds for $t$ and prove it for $t+1$:
\begin{equation}\label{eq:ind-step}
\begin{aligned}
& \mathbb{E}\!\left[\frac{r_{m'-t}}{2^t}\cdot
\prod_{i=1}^{t} g\!\left(\frac{r_{m'-t}}{2^i}\right)\right] \\
& \quad = \mathbb{E}\!\left[\mathbb{E}\!\left[
 \frac{r_{m'-t}}{2^t} \cdot
\prod_{i=1}^{t} g\!\left(\frac{r_{m'-t}}{2^i}\right)
\mid r_{m'-t-1}\right]\right].
\end{aligned}
\end{equation}

Using Lemma~\ref{lem:mult-conv-increas} we have $h(x) :=  \frac{x}{2^t} \cdot \prod_{i=1}^{t} g\!\left(\frac{x}{2^i}\right)$ is a convex function. 
Using Lemma~\ref{lem:conv-rv} this indicates that 
\[
\begin{aligned}
& \mathbb{E}\!\left[
h(r_{m' - t}) 
\mid r_{m'-t-1}\right] \\
& \quad \le
\frac{2\mathbb{E}\!\left[r_{m' - t} \mid r_{m'-t-1}\right] }{r_{m' - t - 1}} h \left(\frac{r_{m' - t - 1}}{2}\right) \\
& \quad \le
\frac{2 f \left(\frac{r_{m' - t - 1}}{2}\right)}{r_{m' - t - 1}} h \left(\frac{r_{m' - t - 1}}{2}\right) \\
& \quad = g\!\left(\frac{r_{m'-t-1}}{2}\right) 
\frac{r_{m'-t-1}}{2^{t+1}}
\prod_{i=1}^{t} g\!\left(\frac{r_{m'-t-1}}{2^{i+1}}\right) \\
& \quad = \frac{r_{m'-t-1}}{2^{t+1}}
\prod_{i=1}^{t + 1} g\!\left(\frac{r_{m'-t-1}}{2^{i}}\right)
\end{aligned}
\]
This completes the induction step. Setting $t = m'$ yields
\[
\frac{\mathbb{E}[r_{m'}]}{2}
\le 2^{-m' - 1} \prod_{i=1}^{m'} g(2^{-i}),
\]
which completes the proof.
\end{proof}

\ActiveProbThreshLower*
\begin{proof}
Let $I_0 = [0,1]$. At each step $t$, suppose $I_t = [a_t,b_t]$ contains $\theta$, with length $r_t=b_t-a_t$. Without loss of generality, let the learner's query  be $x_t \le (a_t+b_t)/2$, and define the oracle’s response distribution as
\[
x'_t = 
\begin{cases}
b_t, & \text{with prob. } g(r_t/4),\\
\theta, & \text{otherwise.}
\end{cases}
\]
where $\theta \in \left[\frac{a_t + 3b_t}{4}, b_t\right]$. Then for all such $\theta$,
\[
\begin{aligned}
\mathbb{E}[|x'_t - \theta|]
&= g\left( \frac{r_t}{4} \right)\,|b_t-\theta|  \le g \left( \frac{r_t}{4} \right) \frac{r_t}{4} \\
& \le f\left( \frac{r_t}{4} \right)
\le f(|x_t - \theta|),
\end{aligned}
\]
so this oracle distribution is valid. Define $I_{t+1} = [\theta]$ if $x'_t=\theta$, and 
$I_{t+1}=\left[\frac{a_t + 3b_t}{4}, b_t\right]$ otherwise. Let $I^*_m$ denote the interval after $m$ rounds where $\theta$ was not yet returned. If $h$ is the learner’s output, let
\[
\theta = \argmax_{z \in \{a_m, b_m\}} \Pr[h(x) \ne \mathbbm{1}\{x \le z\}].
\]
In those scenarios, the learner’s error is at least $\frac{|I^*_m|}{2} = \frac{1}{2\cdot 4^m}$, and the probability that $|I_m|>0$ is $\prod_{i=1}^m g(4^{-i})$. Therefore, the expected error of any learner is at least
\[
2^{-(2m+1)} \prod_{i=1}^m g(4^{-i}),
\]
which completes the proof.
\end{proof}

\ActiveProbThreshPrLower*
\begin{proof}
    Observe that Theorem~\ref{thm:active-prob-thresh-lower} also implies that no learner can achieve error of $\frac{1}{2^{-(2m + 1)}}$ with probability more than $\prod_{i=1}^m g(4^{-i})$, which completes the proof.
\end{proof}

\section{Missing proofs from Section~\ref{sec:active-prob-acc-req}}
\ActiveProbThreshAccReq*
\begin{proof}
    Starting from $I_0 = [0, 1]$, let $I_t = [a_t, b_t]$ denote the interval of valid positions for $\theta$ after $t$ queries. Suppose $x'_{a_t}$ is the contrastive example for $a_t$ and $x'_{b_t}$ is the contrastive example for $b_t$. Then $\mathbb{E}[x'_{a_t} - x'_{b_t}] \leq f(\theta - a_t) + f(b_t - \theta) \leq 2f(b_t - a_t)$. By Markov’s inequality, $\Pr\!\left(x'_{a_t} - x'_{b_t} > 4f(b_t - a_t)\right) \leq \tfrac{1}{2}$.

    Suppose the learner queries $a_t$ and $b_t$ until $x'_{a_t} - x'_{b_t} \leq 4f(b_t - a_t)$. Let $M_t$ be the random variable representing the number of times the learner needs to do this. Then for $t' = t + 2M_t$, we have $|I_{t'}| \leq 4f(|I_t|)$. Continue this process until $|I_t| \leq 2 \varepsilon$, at which point the learner outputs the midpoint of the interval, guaranteeing an error less than $\varepsilon$. 

    Note that $M_t$ is dominated by a geometric random variable with parameter $1/2$. This implies $\mathbb{E}[M_t] = 2$, which immediately yields $\mathcal{N}_{\act} \left(\thresh, \CEprob, \varepsilon\right) \leq 4 \STOP{4f}{1}{2\varepsilon}.$
\end{proof}

\begin{lemma}[Multiplicative Chernoff bounds \cite{motwani1996randomized}] \label{lem:mult-chern-bound} Let $X_1, \ldots, X_m$ be independent random variables drawn according to some distribution $\mathcal{D}$ with mean $p$ and support included in $[0,1]$. Then, for any $\gamma \in\left[0, \frac{1}{p}-1\right]$, the following inequality holds for $\widehat{p}=\frac{1}{m} \sum_{i=1}^m X_i$:

$$
\begin{aligned}
& \mathbb{P}[\widehat{p} \geq(1+\gamma) p] \leq e^{-\frac{m p \gamma^2}{3}} \\
& \mathbb{P}[\widehat{p} \leq(1-\gamma) p] \leq e^{-\frac{m p \gamma^2}{2}}
\end{aligned}
$$
\end{lemma}

\activeProbThreshPrUpper* 
\begin{proof}
We follow the notation of Theorem~\ref{thm:active-prob-thresh-acc-req}.
At each time step $t$, let $I_t = [a_t,b_t]$ denote the interval of valid locations
for $\theta$.
Define the indicator random variable $A_t$ by
\[
A_t := \mathbbm{1}\!\left\{
x'_{a_t} - x'_{b_t} \leq 4f(b_t-a_t)
\right\}.
\]
By the argument in the proof of Theorem~\ref{thm:active-prob-thresh-acc-req},
we have $\mathbb{E}[A_t] \geq \tfrac{1}{2}$.

Whenever $A_t = 1$, querying $a_t$ and $b_t$ results in an update satisfying
\[
|I_{t+2}| \leq 4f(|I_t|).
\]
Moreover, the sequence $(|I_t|)_{t\ge 0}$ is non-increasing.
Therefore, if among $A_{1}, A_{3}, ..., A_{2 \lfloor \frac{m + 1}{2} \rfloor - 1} $ there are at least
\(\STOP{4f}{1}{2\varepsilon}
\)
 ones, then after $m$ queries we
necessarily have $|I_m| \leq 2\varepsilon$, and hence the learner’s output has error
at most $\varepsilon$.

Now consider the sum
\[
S_m := \sum_{t=1}^{\lfloor (m+1)/2 \rfloor} A_{2t-1}.
\]
Applying the multiplicative Chernoff bound (Lemma~\ref{lem:mult-chern-bound})
with $\gamma = 1/2$, for $m \geq 8 \STOP{4f}{1}{2\varepsilon} + 16 \ln \left(\frac{1}{\delta} \right)$ we obtain
$$
Pr \left [S_m < \STOP{4f}{1}{2\varepsilon}\right] \leq e^{-\frac{m}{16}  } \leq \delta
$$
With probability at least $1-\delta$, the learner thus achieves error at most
$\varepsilon$ after $m$ queries, completing the proof.
\end{proof}

\HSProbHelper*
\begin{proof}
    Let $x'$ be the contrastive example for $x$. Without loss of generality, assume $\ell(x) = 1$. Choose a point $z$ in the line fragment between $x$ and $x'$ such that 
    $\|z - x\| = (\mathbbm I + 2f)^{-1}(\|x - x'\|)$. This implies
    $$\|z - x'\| = 2f(\|z - x\|).$$
    If the label of $z$ is $0$, discard $x'$ and query $x$ for another contrastive example. Define $z^*$ as the first $z$ with label $1$, and let $M$ denote the random variable representing the number of repetitions of this process.

    Since $\ell(z^*) = \ell(x)$ we have $\dist(z^*, \ell) \leq \dist(x, \ell)$ which completes property (i); it remains to show property (ii) and to bound $\mathbb{E}[M]$. Using Markov’s inequality, with probability at least $1/2$ we have $\dist(x', \ell) \leq 2f(r)$. In that case, the label of $z$ is guaranteed to be $1$. Thus, in each iteration, with probability $1/2$ the label of $z$ is $1$, implying $\mathbb{E}[M] = 2$.

    Finally, we establish property (ii) for $z^*$. By Markov’s inequality, the probability that $\|x - x'\| > 2r$ is at most $\tfrac{f(r)}{r} \leq \tfrac{1}{4}$. Hence, given $\ell(z^*) = 1$, the probability that $\|x - x'\| > 2r$ is at most $1/2$. Observe that $x'$ and $x$ lie on opposite sides of $\ell$. Therefore, as long as $\|x - x'\| \leq 2r$, we have
    \begin{equation} \label{eq:HS-prob-helper}
    \begin{aligned}
        \dist(z^*, \ell) & \leq \frac{\|z^* - x'\|}{\|z^* - x\|} r  \\
        & = \frac{2f(\|z^* - x\|)}{\|z^* - x\|} r \\
        & \leq \frac{2f(\|x' - x\|)}{\|z^* - x\|} r \\
        & \leq 2f(2r) \frac{r}{\|z^* - x\|}
    \end{aligned}
    \end{equation}
    Next note that since $f \leq \frac{\mathbb I}{4}$ we have
    $$
    \|z^* - x\| = (\mathbb I + 2f)^{-1} ( \|x - x'\|) \geq \frac{2 \|x - x'\|}{3} \geq \frac{2r}{3}.
    $$
    Plugging this into \eqref{eq:HS-prob-helper} completes the proof.
    
\end{proof}

\ActiveProbUpperHS*
\begin{proof}
   The learner divides the training process into multiple sub-phases. By initializing $z_0 = \mathbf{0}$, in each subroutine $t = 1, 2, \dots$, the learner follows the procedure introduced in Lemma~\ref{lem:HS-prob-helper} with $x$ set to $z_{t - 1}$, and defines $z_{t}$ as the returned $z^*$ of the procedure. 
   
   Denote $z'_{t}$ as the contrastive example corresponding to $z_t$. For each $t$, define $C_t$ to be the halfspace whose boundary hyperplane is perpendicular to $z'_t - z_t$ and passes through $z'_t$. Suppose the learner’s output is always $C_t$ for the most recent subroutine $t$.
   
   For every $z_t$ denote $z^{\proj}_t$ as the projection of $z_t$ into $\ell$. Let $T$ be the random variable representing the first subroutine such that 
   $$
   \frac{2^k \|z'_T - z^{\proj}_T\|}{ \dist(z_T, \ell)} \leq \varepsilon.
   $$
   By Lemma~\ref{lem:err-dist} (ii), this implies $\err(C_T, \ell) \leq \varepsilon$. We now prove that $\mathbb{E}[T] \leq 2\STOP{\tilde f}{\frac{1}{2}}{\varepsilon'} + 2$. Combining this bound with Lemma~\ref{lem:HS-prob-helper} completes the proof.
   
    For every $t$, let $M_t$ denote the random variable $t' - t$, where $t'$ is the first time such that $\dist(z_{t'}, \ell) \leq \tilde f(\dist(z_{t' - 1}, \ell))$. Since $M_t$ is stochastically dominated by a geometric random variable with parameter $1/2$, we have $\mathbb{E}[M_t] \leq 2$. 
    Let $T'$ represent the first subroutine satisfying $\dist(z_{T'}, \ell) \leq \varepsilon'$. Then, $\mathbb{E}[T'] \leq 2\STOP{\tilde f}{\frac{1}{2}}{\varepsilon'}$.
    
     Finally, for each $t \geq T'$, due to Markov's inequality we have 
     \begin{equation} \label{eq:active-prob-HS-upper}
         \Pr\left[\|z_t^{\proj} - z'_t\| > 2f(\dist(z_t, \ell))\right] \leq \tfrac{1}{2}.
     \end{equation}
     Consequently, in expectation, with two additional subroutines after $T'$ the learner reaches a $z_T$ such that
     $$
     \begin{aligned}
         \frac{2^{k} \|z^{\proj}_T - z'_T\|}{\dist(z_T, \ell)} & \leq   \frac{2^{k+1} f(\dist(z_T, \ell))}{\dist(z_T, \ell)}  \leq 2^{k - 1} \dist(z_T, \ell)^{c} \\
         &
         \leq 2^{k - 1} {\varepsilon'}^c = \varepsilon.
     \end{aligned}
     $$
\end{proof}

\activeProbHSPrUpper* 
\begin{proof}
    We first show that with at most $T = 4 \STOP{\tilde f}{\frac{1}{2}}{\varepsilon'} + 17 \log_2 \frac{3}{\delta}$ subroutines as defined in Theorem~\ref{thm:active-prob-upper-HS} the learner will have error less than $\varepsilon$ with probability at least $\frac{2 \delta}{3}$.

    Define the indicator random variable $A_t$ as
    $$A_t := \mathbbm{1} \{ \dist(z_{t}, \ell) \leq \tilde f(\dist(z_{t - 1}, \ell)) \}$$
    for $z_t$ defined in Theorem~\ref{thm:active-prob-upper-HS}. Denote $T' = 4 \STOP{\tilde f}{\frac{1}{2}}{\varepsilon'} + 16 \ln \frac{3}{\delta}$. Note that whenever among $A_1, ..., A_{T'}$ at least $\STOP{\tilde f}{\frac{1}{2}}{\varepsilon'}$ number of ones, then we necessarily must have
    $$
    \dist(z_{T'}, \ell) \leq \varepsilon'\,.
    $$

    Consider the sum $S^1_{T'} := \sum_{t = 1}^{T'} A_t$. In Theorem~\ref{thm:active-prob-upper-HS} we showed $\mathbb{E}[A_t] = \frac{1}{2}$. Therefore, using  multiplicative Chernoff bounds (Lemma~\ref{lem:mult-chern-bound}) with $\gamma = \frac{1}{2}$ we derive
    $$
    \Pr\left[S^1_{T'} < \STOP{\tilde f}{\frac{1}{2}}{\varepsilon'}\right] \leq e^{-\frac{T'}{16}} \leq \frac{\delta}{3}\,.
    $$

    Moreover, since for all $t$ with probability $\frac{1}{2}$ we have
    $$
    \|z^{\proj}_t - z'_t\| \leq 2f(\dist(z_t, \ell))\,.
    $$
    Then
    $$
    \begin{aligned}
        \Pr\left[ \forall t \in [T' + 1, T]: [\|z^{\proj}_t - z'_t\| > 2f(\dist(z_t, \ell))\right] \\
        < \frac{1}{2^{T - T'}} < \frac{\delta}{3}\,.
    \end{aligned}
    $$
    Thus, using Theorem~\ref{thm:active-prob-upper-HS} with probability at least $\frac{2\delta}{3}$, the learner will have error less than $\varepsilon$.

    It remains to show that with probability $1 - \frac{\delta}{3}$ with $8T$ queries the number of subroutines will be at most $T$.  Note that the number of subroutines is the number of times $z$ defined in Lemma~\ref{lem:HS-prob-helper} has label 1 in $8T$ queries. 

    Denoting $K$ as the number of subroutines, and using multiplicative Chernoff bounds (Lemma~\ref{lem:mult-chern-bound}) with $\gamma = \frac{1}{2}$ we derive
    $$
    \Pr[K \leq T] \leq e^{\frac{-4T}{16}} \leq \frac{\delta}{3}\,,
    $$
    which completes the proof.
\end{proof}

\section{Missing proofs from Section~\ref{sec:passive-prob}}

\PassiveProbThresh*
\begin{proof}
    \textbf{Lower bound:} 
    Theorem~\ref{thm:passive-noisy-min-thresh} implies that for any $\varepsilon < \min\!\left( \frac{f(\tfrac{1}{4})}{4}, \tfrac{1}{8}\right)$, if 
    $m \leq \frac{1}{4\min(f^{-1}(4\varepsilon), 2\varepsilon)}$,
    then any learner interacting with the deterministic AMDM incurs error at least $2\varepsilon$ with probability $1/2$. Thus, its expected error is at least $\varepsilon$. Since learning with the probabilistic AMDM is harder, the same lower bound also applies to $\CEprob$.

    \textbf{Upper bound:}
    Let $\varepsilon' = \min(f^{-1}(\varepsilon / 2), \varepsilon / 2)$. As in Theorem~\ref{thm:passive-noisy-min-thresh}, let $x^*$ be the largest positively labeled $x_i$ (assuming without loss of generality that it exists), and let $\tilde x^*$ be its contrastive example. If $f(\varepsilon) \leq \varepsilon$, the learner outputs the threshold corresponding to the contrastive example of $x^*$; otherwise, it outputs the threshold corresponding to $x^*$.

    We showed in Theorem~\ref{thm:passive-noisy-min-thresh} that with probability at least $1 - \varepsilon/2$, for 
    $m \geq \frac{\log_2 (2 / \varepsilon)}{\varepsilon'}$ we have $x^* \geq \theta-\varepsilon'$. Therefore,
    $$
        \mathbb{E}[\tilde x^* - \theta \mid x^*] \leq f(\varepsilon')
    $$
    Thus, given a fixed set of primary examples, the learner makes an expected error of at most $\varepsilon/2$ with probability at least $1 - \varepsilon/2$. Finally, Lemma~\ref{lem:highprob-to-expect} implies that the learner's expected error is at most $\varepsilon$.
\end{proof}

\PassiveProbHS*
\begin{proof}
For any $x_i$ and $x'_i$, define $C_i$ as the homogeneous half-space perpendicular to $x'_i - x_i$. Let
$i^* := \argmin_i \dist(x_i, \ell)$
and
$\hat i := \argmin_i \|x_i - x'_i\|$.
Our output hypothesis is $\hat C = C_{\hat i}$ (the same algorithm as in Theorem~\ref{thm:passive-noisy-min-HS}).

Lemma~\ref{lem:err-dist} (i) implies that
\begin{equation}\label{eq:passive-prob-HS-1}
    \begin{aligned}
    \err(\hat C, \ell)  
    & \leq \frac{1}{\pi} \cos^{-1} \left( \frac{ \dist(x_{\hat i}, \ell)}{\|x_{\hat i} - x'_{\hat i}\|} \right)  \\
    & \stackrel{(i)}{\leq} \frac{1}{\pi} 
    \cos^{-1} \left( \frac{ \dist(x_{i^*}, \ell)}{\|x_{\hat i} - x'_{\hat i}\|} \right) \\
    & \stackrel{(ii)}{\leq} \frac{1}{\pi}\cos^{-1} \left( \frac{ \dist(x_{i^*}, \ell)}{\|x_{i^*} - x'_{i^*}\|} \right),
    \end{aligned}
\end{equation}
where (i) follows from the definition of $i^*$, and (ii) follows from the definition of $\hat i$.
Let $x^{\proj}_{i^*}$ be the projection of $x_{i^*}$ onto $\ell$. Since the angle $\angle(x'_{i^*}, x^{\proj}_{i^*}, x_{i^*})$ is at least $\pi/2$, we have
\begin{equation}\label{eq:passive-prob-HS-2}
\begin{aligned}
    & \cos^{-1} \left( \frac{ \dist(x_{i^*}, \ell)}{\|x_{i^*} - x'_{i^*}\|} \right) \\
    &\quad \leq \cos^{-1} \left( \frac{ \dist(x_{i^*}, \ell)}{\sqrt{\dist(x_{i^*}, \ell)^2 + \|x_{i^*} - x'_{i^*}\|^2}} \right)  \\
    &\quad = \tan^{-1} \left( \frac{\|x_{i^*} - x'_{i^*}\|}{\dist(x_{i^*}, \ell)} \right) \leq  \frac{\|x_{i^*} - x'_{i^*}\|}{\dist(x_{i^*}, \ell)} .
\end{aligned}
\end{equation}
Combining \eqref{eq:passive-prob-HS-1} and \eqref{eq:passive-prob-HS-2} yields
$$
\err(\hat C, \ell)\leq \frac{\|x_{i^*} - x'_{i^*}\|}{\pi \dist(x_{i^*}, \ell)} .
$$

We now bound the expectation of $\err(\hat C, \ell)$ conditioned on a fixed set of primary examples. In particular,
\begin{equation} \label{eq:passive-prob-HS-3}
\begin{aligned}
    \mathbb{E}[\err(\hat C, \ell) \mid x_1, \ldots, x_m]
    & \leq  \frac{\mathbb{E}[\|x_{i^*} - x'_{i^*}\| \mid x_{i^*}]}{\pi\,\dist(x_{i^*}, \ell)} \\
    & \leq \frac{f(\dist(x_{i^*}, \ell))}{\pi\, \dist(x_{i^*}, \ell)} \\
    &= \frac{g(\dist(x_{i^*}, \ell))}{\pi}.
\end{aligned}
\end{equation}

Similar to the argument in Theorem~\ref{thm:passive-noisy-min-HS}, for
$m \geq \frac{\ln(2 / \varepsilon)}{2 g^{-1}\left(\frac{\pi \varepsilon}{2}\right)}$
we have, with probability at least $1 - \varepsilon / 2$,
$$
\dist(x_{i^*}, \ell) \leq g^{-1}\left(\frac{\pi \varepsilon}{2}\right).
$$
Combining this with \eqref{eq:passive-prob-HS-3} and Lemma~\ref{lem:highprob-to-expect} gives
$$
\mathbb{E}[\err(\hat C, \ell)] \leq \varepsilon,
$$
which completes the proof.
\end{proof}
    
\end{document}